\documentclass{article}

\usepackage[preprint]{neurips_2026}

\usepackage[utf8]{inputenc} 
\usepackage[T1]{fontenc}    
\usepackage{hyperref}       
\usepackage{url}            
\usepackage{booktabs}       
\usepackage{amsfonts}       
\usepackage{nicefrac}       
\usepackage{microtype}      
\usepackage{xcolor}         
\usepackage{amsmath}
\usepackage{amssymb}
\usepackage{mathtools}
\usepackage{amsthm}
\usepackage{placeins}
\usepackage{algorithm}
\usepackage{algpseudocode}
\usepackage{graphicx}
\usepackage{makecell}
\usepackage{subcaption}
\usepackage{listings}
\usepackage{xcolor}

\lstdefinestyle{promptstyle}{
    backgroundcolor=\color{gray!10},   
    basicstyle=\ttfamily\small,        
    breakatwhitespace=false,         
    breaklines=true,                   
    captionpos=b,                    
    keepspaces=true,                 
    numbers=none,                      
    keywordstyle=\color{blue},       
    stringstyle=\color{green!70!black},
    commentstyle=\color{gray}
}
\theoremstyle{plain}
\newtheorem{theorem}{Theorem}[section]

\newtheorem{lemma}[theorem]{Lemma}

\theoremstyle{definition}

\newtheorem{assumption}[theorem]{Assumption}
\theoremstyle{remark}

\title{ATLAS: A Multi-LLM Training Framework for EvoDPO with Adaptive Reference Evolution}

\author{
Ujin Jeon \\
School of Electrical and Computer Engineering \\
Purdue University \\
West Lafayette, United States \\
\texttt{ujeon@purdue.edu}
\And
Jiyong Kwon \\
School of Mechanical Engineering \\
Purdue University \\
West Lafayette, United States \\
\texttt{kwon165@purdue.edu}
\And
Madison Ann Sullivan \\
Department of Mathematics \\
Purdue University \\
West Lafayette, United States \\
\texttt{sulli434@purdue.edu}
\And
Caleb Eunho Lee \\
Department of Computer Science \\
Purdue University \\
West Lafayette, United States \\
\texttt{lee4929@purdue.edu}
\And
Guang Lin\thanks{Corresponding author: \texttt{guanglin@purdue.edu}} \\
Department of Mathematics and Mechanical Engineering \\
Purdue University \\
West Lafayette, United States \\
\texttt{guanglin@purdue.edu}
}

\begin{document}

\maketitle

\begin{abstract}
Recent multi-LLM agent systems have shown promising capabilities for automated problem-solving, yet they predominantly rely on frozen agents or static fine-tuning pipelines. To address this limitation, our primary contribution is ATLAS (Adaptive Task-distributed Learning for Agentic Self-evolution), a multi-agent framework where specialized meta-agents collaboratively train and refine an active agent toward a domain-specific policy. A core challenge in iterative preference learning within these pipelines is the reliance on fixed reference models, which typically leads to overly conservative updates or training stagnation. To overcome this, the framework's algorithmic engine utilizes Evolving Direct Preference Optimization (EvoDPO). EvoDPO employs an inspection agent to perform adaptive, proxy-KL gated reference policy updates based on continuous training telemetry. We evaluate this full framework across a diverse set of challenging environments—including non-stationary contextual bandits, partial differential equations (PINNs), and combinatorial optimization tasks (TSP, Bin Packing). Through comparison against fixed-reference, adaptive-reference, and external automated-discovery baselines, our results suggest that ATLAS combines supporter-driven exploration with EvoDPO-driven stability to improve long-horizon evaluator-driven self-improvement.

\end{abstract}

\section{Introduction}

Fine-tuned large language model (LLM) agents are broadly used as components in agentic systems for complex problem solving, scientific computing, and code-generation workflows. Recent frameworks show that collections of specialized agents, combined with experience banks, routing, or collaborative coordination, improve task performance and reliability by optimizing search, planning, decision-making, strategies, or data generation \cite{liu2025lessonslearnedmultiagentframework,jiang2025agenticscimlcollaborativemultiagentsystems,wu2025automatedcodedevelopmentpde,he2025langpinnlanguagephysicsinformedneural,chen2025warriormathenhancingmathematicalability}. In parallel, multi-agent infrastructure and benchmarking efforts have made it easier to orchestrate such systems and study their behaviors at scale \cite{chen2023agentversefacilitatingmultiagentcollaboration,wu2023autogenenablingnextgenllm,chen2024autoagentsframeworkautomaticagent,huang2025competing,becker2025mallmmultiagentlargelanguage}.

Despite these advances, many existing approaches predominantly treat LLM agents as frozen optimizers and focus on improving optimization efficiency rather than explicitly studying how to develop an LLM agent through iterative model updates under realistic non-stationary and long-horizon training. To bridge this gap, our work introduces a framework for robust self-evolution via controlled reference management.

In this work, we propose ATLAS (Adaptive Task-distributed Learning for Agentic Self-evolution), a framework targeting specialized Evolving Agent formation via systematic model development. Rather than relying on a single static agent, we fine-tune a set of supporter LLM agents that collectively cultivate and refine a domain-specific Evolving Agent. The supporter agents are assigned complementary roles that support the learning dynamics: (i) exploration support to suggest diverse exploration strategies (ii) fine-tuning supervision to stabilize preference-based updates and prevent excessive deviation, and (iii) reference-policy inspection to decide when and how to renew the reference policy for well-aligned adaptation. By explicitly separating these tasks, the training loop prioritizes progressive agent improvement rather than merely search or planning enhancements, and it remains compatible with scalable multi-agent execution \cite{xie2024aimetropolisscalinglarge}.

Our learning mechanism leverages preference optimization as a practical alternative to Reinforcement Learning from Human Feedback (RLHF) pipelines. Direct Preference Optimization (DPO) enables preference learning without an explicit reward model \cite{rafailov2024directpreferenceoptimizationlanguage}, and recent work extends preference learning beyond simple binary comparison and introduces progressive schemes to handle distribution shift \cite{liao2025tpoaligninglargelanguage,yang2025weightedrewardpreferenceoptimizationimplicit}. Our work aligns with evolutionary or multi-agent fine-tuning approaches that iteratively improve an Evolving Agent using critique or debate signals \cite{surina2025algorithmdiscoveryllmsevolutionary,zhou2025debatereflectdistillmultiagent}. However, a key limitation in many iterative preference-based pipelines is fixed reference policies, which leads to misaligned references, overly conservative updates, or stagnation.

To address this, we introduce Evolving DPO (EvoDPO), a preference-optimization loop with telemetry-driven fine-tuning control and adaptive reference management. At each fine-tuning phase, a strategist agent tunes DPO hyperparameters based on training diagnostics. In parallel, EvoDPO updates its policy using the DPO algorithm, while modeling the reference policy as a phase-indexed variable. Specifically, the Evolving Agent proposes a Kullback-Leibler (KL) regularized reference policy to be used in the next phase. Then, a policy-inspector agent applies a proxy-KL gate to decide whether to promote the proposed policy as the next reference or to keep the existing one. This integration of adaptive preference optimization and task-distributed agents prevents stagnation from stale references while avoiding instability from overly aggressive reference adaptation, enabling robust long-horizon self-evolution. The resulting mechanism is complementary to recent multi-agent preference-alignment approaches \cite{lyu2025macpoweaktostrongalignmentmultiagent}.

We analyze EvoDPO through an idealized non-stationary preference contextual bandit, where the regret decomposition separates reference-induced bias from learning error under drift.

Empirically, we evaluate ATLAS on four executable optimization domains: non-stationary contextual bandits, PINN loss reweighting for Burgers' equation, TSP, and Bin Packing. We compare against EvoTune~\cite{surina2025algorithmdiscoveryllmsevolutionary}, AutoGen~\cite{wu2023autogenenablingnextgenllm}, and FunSearch-style evaluator-driven program search.

Our contributions are as follows:

(i) We introduce ATLAS, a role-specialized multi-LLM training framework for evaluator-driven executable self-improvement. ATLAS separates candidate exploration, fine-tuning control, and reference-policy inspection rather than treating multi-agent interaction as a single monolithic search procedure.

(ii) We propose EvoDPO, an adaptive-reference variant of DPO that treats the reference policy as a phase-indexed variable and promotes a new reference only when it improves evaluator score while satisfying an empirical proxy-KL constraint.

(iii) We show that ATLAS gains arise from two complementary mechanisms: supporter agents improve the candidate distribution by expanding the search space, while EvoDPO stabilizes iterative preference optimization by mitigating stale-reference stagnation and uncontrolled reference drift.

(iv) We evaluate ATLAS across four executable optimization domains, including non-stationary bandits, PINN loss reweighting, TSP, and Bin Packing, and include ablations that isolate fixed-reference training, adaptive-reference training without supporters, and simplified reference-promotion alternatives.

\section{Proposed Methods}
\label{method}

\subsection{ATLAS Architecture and Role-Specialization}

We propose ATLAS as a collaborative, multi-agent hierarchy designed to decouple code exploration from optimization safety. Rather than relying on a single generalist model, we distribute distinct responsibilities across specialized LLMs. At the core of ATLAS is the Evolving Agent (Llama-3.2-1B), which is the only model updated by preference optimization. At phase $k$, the Evolving Agent receives prompts $x_t\sim\mathcal{X}$ and generates candidate executable programs or policies.

To construct $\mathcal{D}_k$, each candidate is evaluated by the corresponding task evaluator---simulating the non-stationary bandit, computing the PINN validation metric, or executing the heuristic solver for TSP and Bin Packing---to obtain a scalar utility score. We then form pairwise preferences by assigning the higher-scoring candidate as $y^+$ and the lower-scoring candidate as $y^-$. A pair is added to $\mathcal{D}_k$ only if its score gap exceeds a phase-dependent margin threshold. Importantly, supporter LLMs do not assign preference labels; they only propose candidate modifications or training-control suggestions, while all preference labels are induced by executable evaluator scores.

To guide this generation, an Exploration Supporter (gpt-oss-120B) performs static analysis on the Evolving Agent's output, injecting domain-specific architectural patches (e.g., dynamic normalization, Huber loss) to prevent mode collapse. During the EvoDPO update, a Fine-Tuning Strategist (Qwen 3-32B) regulates the learning dynamics by monitoring phase-level training telemetry, such as score distributions, to dynamically modulate the DPO inverse-temperature $\beta_{\mathrm{DPO}}$ and pair-selection thresholds. Finally, to ensure stable progression, a Policy Inspector (Llama4-latest) operates as a strict safety gate for adaptive reference management. It evaluates the proposed KL-regularized policy $\pi^{kl}_k$ against the current reference $\pi_{\mathrm{ref},k}$ using an empirical trajectory-conditioned proxy $\widehat{\mathrm{KL}}_k$. This inspector enforces a proxy-KL gate, accepting updates only if the candidate demonstrates sufficient utility improvement ($\Delta \hat{S}_k \geq \epsilon_s$) without violating the predefined drift budget ($\delta_H$).

For full reproducibility, the exact system prompts, API configurations, and concrete examples of architectural patches generated by the supporter agents are detailed in Appendix C.

\begin{figure*}[t]
\centering
  \includegraphics[width=\textwidth,trim={0 40 0 40},clip]{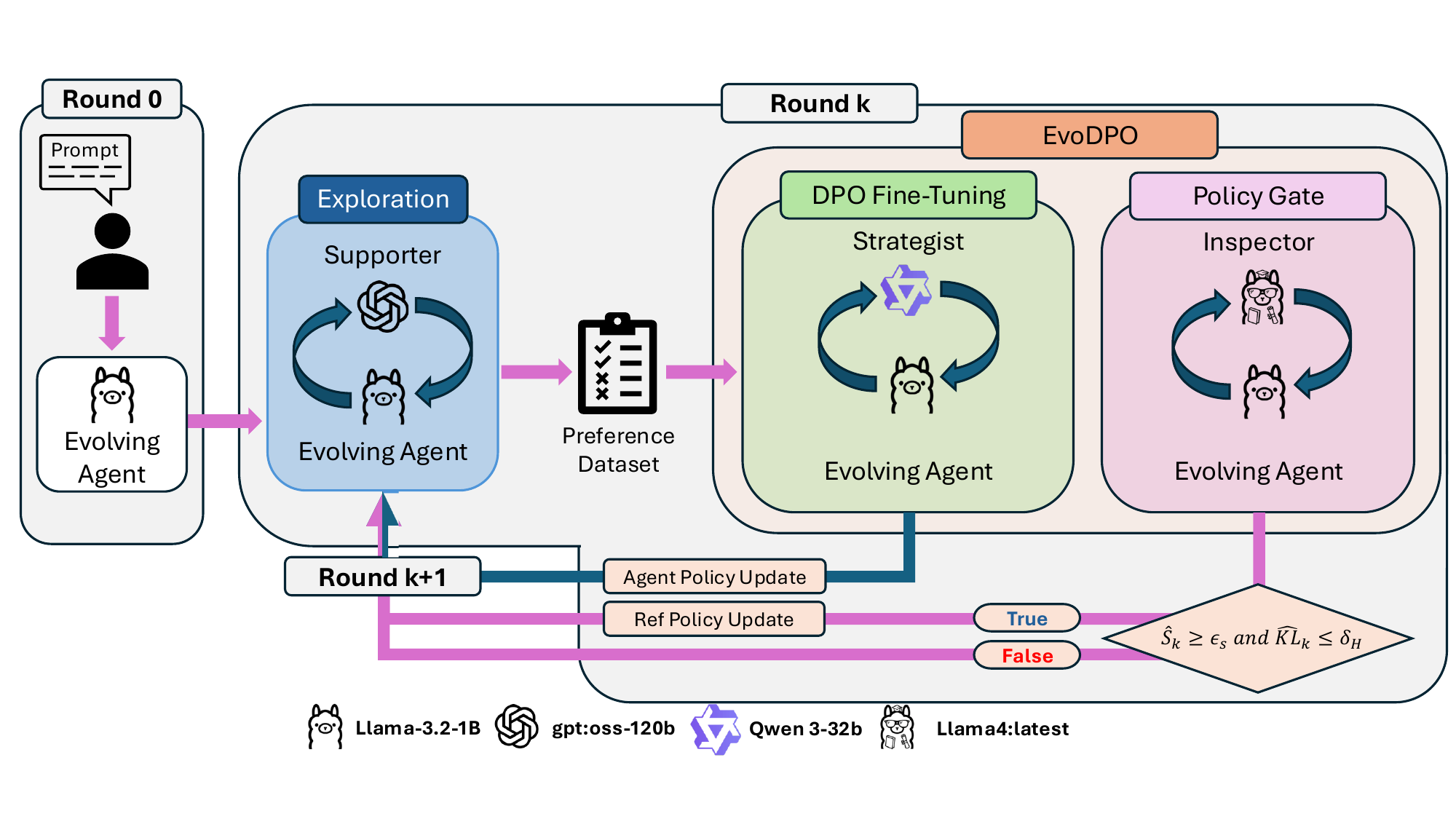}
  \caption{\emph{ATLAS workflow.} ATLAS alternates between supporter-guided candidate exploration and EvoDPO updates with strategist-guided fine-tuning and proxy-KL-gated reference promotion.}
  \label{fig:architecture}
\end{figure*}

\subsection{Evolving DPO algorithm}
In this section, we introduce our core algorithm, Evolving DPO (EvoDPO), that performs preference optimization with KL-regularized reference improvement.

\paragraph{Direct Preference Optimization.}Direct preference optimization (DPO)
\cite{rafailov2024directpreferenceoptimizationlanguage} is an efficient alternative to reward-model-based RLHF.  It optimizes a policy from pairwise preferences and avoids explicit reward learning and value-function estimation. Given a preference dataset $D$ of $(x, y^+, y^-)$, DPO minimizes
\begin{equation}
\label{eq:vanilla-dpo}
\begin{aligned}
L_{\mathrm{DPO}}
&= -\mathbb{E}_{(x,y^+,y^-)\sim D}\Biggl[
\log \sigma\Biggl(
\beta_{DPO} \log \frac{\pi_{\theta}(y^+ \mid x)}{\pi_{\mathrm{ref}}(y^+ \mid x)}  -\beta_{DPO} \log \frac{\pi_{\theta}(y^- \mid x)}{\pi_{\mathrm{ref}}(y^- \mid x)}
\Biggr)\Biggr]
\end{aligned}
\end{equation}
where $\pi_{\mathrm{ref}}$ is a fixed reference policy and $\beta_{DPO}>0$ controls the strength of the preference margin.

A practical assumption behind \eqref{eq:vanilla-dpo} is that the preference data distribution is reasonably aligned with the reference policy: many public preference datasets are collected from, or near, a supervised fine-tuning (SFT) policy, and optimization is most stable when the likelihood ratios in \eqref{eq:vanilla-dpo} are not evaluated far off-distribution. In long-horizon phase-indexed loops, however, a fixed $\pi_{\mathrm{ref}}$ can become outdated as the Evolving Agent improves its policy and dataset. In this case, the mismatch between $D_k$ and $\pi_{\mathrm{ref}}$ can yield slow adaptation and misalignment.

\paragraph{Evolving DPO.}
To address reference-data mismatch in long-horizon self-improvement, we introduce Evolving DPO (EvoDPO), which replaces the fixed reference with a phase-indexed reference $\pi_{\mathrm{ref},k}$ that updates across phases. At phase $k$, the Evolving Agent is updated by optimizing a DPO objective anchored at the current reference:
\begin{equation}
\label{eq:adaptive-dpo}
\begin{aligned}
L_{\mathrm{EvoDPO}}
&= -\mathbb{E}_{(x,y^+,y^-)\sim D_k}\Biggl[
\log \sigma\Biggl(
\beta_{DPO} \log \frac{\pi_{\theta}(y^+ \mid x)}{\pi_{\mathrm{ref},k}(y^+ \mid x)} -\beta_{DPO} \log \frac{\pi_{\theta}(y^- \mid x)}{\pi_{\mathrm{ref},k}(y^- \mid x)}
\Biggr)\Biggr]
\end{aligned}
\end{equation}
yielding an updated Evolving Agent policy $\pi_{\theta,k+1}$. The reference is then updated conservatively by the reference management.

\paragraph{Empirical KL estimate in the gate.}
The exact prompt-conditional divergence $D_{KL}$
\[
D_{KL}\big(\pi(\cdot\mid x)\,\|\,\pi_{\mathrm{ref},k}(\cdot\mid x)\big)
\]
between sequence-valued policies is intractable to compute, since it requires marginalization over the space of output sequences. We therefore use a tractable, trajectory-conditioned KL estimator that compares next-token distributions along a representative candidate trajectory while conditioning on the same prompt. Specifically, for each prompt $x \in G_k$, we associate a candidate response $y=(y_1,\ldots,y_{L_y})$ from the phase buffer. In our implementation, $y$ is the preferred candidate $y^{+}$ associated with $x$. For each token position $\ell=1,\ldots,L_y$, define the history-conditional next-token distributions
\[
p_{x,y,\ell}(\cdot) := \pi(\cdot \mid x,y_{<\ell}),
\quad
q_{x,y,\ell}(\cdot) := \pi_{\mathrm{ref},k}(\cdot \mid x,y_{<\ell}),
\]
where $y_{<\ell} = (y_1,\ldots,y_{\ell-1})$ denotes the prefix up to (but excluding) position $\ell$. We define the per-pair token-average KL as
\[
\bar{D}(x,y)
:= \frac{1}{L_y}\sum_{\ell=1}^{L_y}
D_{KL}\bigl(p_{x,y,\ell}\,\|\,q_{x,y,\ell}\bigr),
\]
and the empirical gate statistic as
\begin{equation}
\label{eq:empirical_kl_gate}
\widehat{\mathrm{KL}}_{k}\bigl(\pi\,\|\,\pi_{\mathrm{ref},k}\bigr)
:= \frac{1}{|G_k|}\sum_{(x,y)\in G_k}\bar{D}(x,y).
\end{equation}
In implementation, $D_{KL}(p_{x,y,\ell}\|q_{x,y,\ell})$ is computed directly from model logits via softmax with standard causal alignment, and masking excludes the prompt tokens and averages only over completion tokens.

\paragraph{Inspector score functional.}
At phase $k$, the Policy Inspector decides whether to promote the reference using a randomly sampled gate subset
$G_k$ drawn from the current phase buffer. Each element of $G_k$ is a prompt $x$ together with its associated evaluated candidates and task scores collected during exploration. This design aligns the gate with the on-distribution prompts encountered by the Evolving Agent during self-evolution, while reducing selection adaptivity relative to score-based selection.

Given a policy $\pi$, let $\mathrm{Score}(x,\pi)$ denote the task-specific evaluation score used by the inspector on
prompt $x$. We define the inspector score as the subset average
\[
\widehat{\mathcal{S}}_{k}(\pi)
:=
\frac{1}{|G_k|}\sum_{x\in G_k}\mathrm{Score}(x,\pi).
\]
The inspector uses the score difference relative to the current reference:
\[
\Delta\widehat{\mathcal{S}}_{k}(\pi)
:=
\widehat{\mathcal{S}}_{k}(\pi)-\widehat{\mathcal{S}}_{k}(\pi_{\mathrm{ref},k}),
\qquad
\Delta\widehat{\mathcal{S}}_{k} := \Delta\widehat{\mathcal{S}}_{k}(\pi^{kl}_{k}).
\]

\paragraph{Reference-improvement operator.}
In practice, we approximate the ideal reference-improvement operator by optimizing over a finite candidate set $\mathcal{C}_k \subset \Pi$. At phase $k$, we construct $\mathcal{C}_k$ from (i) a small number of intermediate checkpoints of the Evolving Agent collected during the phase, and (ii) candidate policies produced under the exploration and guidance from the Fine-Tuning Strategist. Thus, $\mathcal{C}_k = \{\pi^{(1)}_k,\ldots,\pi^{(M)}_k\}$, and $\mathcal{C}_k$ always includes the phase-final policy
$\pi_{\theta,k+1}$. With this candidate set, we define the practical reference-improvement operator as
\begin{equation}
\label{eq:practical-kl-operator}
\pi^{kl}_k
:=
\arg\max_{\pi\in\mathcal{C}_k}
\Bigl\{
\widehat{\mathcal{S}}_k(\pi)
-\beta_{\mathrm{ref}}\,\widehat{\mathrm{KL}}_k\bigl(\pi\,\|\,\pi_{\mathrm{ref},k}\bigr)
\Bigr\},
\end{equation}
where $\widehat{\mathcal{S}}_k$ and $\widehat{\mathrm{KL}}_k$ are computed using the phase-$k$ gate subset. Here, $\beta_{\mathrm{ref}} > 0$ is a reference-gating trade-off coefficient that controls how strongly the inspector penalizes reference drift during selection. 

\begin{algorithm}[htbp]
\caption{EvoDPO reference update}
\label{alg:AdaptiveDPO}
\begin{algorithmic}[1]
\State {\bfseries Initialize:} Evolving Agent policy $\pi_{\theta,1}$, Reference $\pi_{\text{ref},1} \leftarrow \pi_{\theta,1}$
\State \textbf{Hyperparameters:} DPO temperature $\beta_{DPO}$, KL budget $\delta_H$, score tolerance $\epsilon_s$
\For{$k = 1, \dots, K$}
    \State Collect $D_k = \{(x_i, y_i^{+}, y_i^{-}, m_i)\}_{i=1}^{N_k}$
    \State Update Evolving Agent policy
    $
    \pi_{\theta,k+1}
    \leftarrow
    \arg\min_{\pi \in \Pi}
    \mathcal{L}_{\text{EvoDPO}}(\pi;\pi_{\text{ref},k},\mathcal{D}_k)
    $
    \State Compute $\pi^{\{k\}}_k$ via Eq.~\eqref{eq:practical-kl-operator}
    \State Compute
    $
    \widehat{\mathrm{KL}}_k(\pi^{\{k\}}_k \,\|\, \pi_{\mathrm{ref},k})
    $
    \If{$\Delta \widehat{\mathcal{S}}_k \geq \epsilon_s$ \textbf{and}
    $\widehat{\mathrm{KL}}_k(\pi^{\{k\}}_k \,\|\, \pi_{\mathrm{ref},k}) \leq \delta_H$}
        \State Accept $\pi_{\text{ref},k+1} \leftarrow \pi^{\{k\}}_k$
    \Else
        \State Reject $\pi_{\text{ref},k+1} \leftarrow \pi_{\text{ref},k}$
    \EndIf
\EndFor
\end{algorithmic}
\end{algorithm}

\paragraph{Reference update with conservative acceptance.}
At phase $k$, the Evolving Agent proposes the candidate reference $\pi^{kl}_{k}$ via Eq. ~\eqref{eq:practical-kl-operator}. Then the Policy Inspector applies a gating rule using (i) a score improvement $\Delta \widehat{\mathcal{S}}_k$ and (ii) a proxy-KL gate measured by the trajectory-conditioned token-level KL estimator $\widehat{\mathrm{KL}}_k$ on the phase-$k$ gate subset $G_k$. The system accepts a reference promotion if
\begin{equation}
\label{eq:trust-region}
\Delta\widehat{\mathcal S}_k \ge \epsilon_s
\quad\text{and}\quad
\widehat{\mathrm{KL}}_k\bigl(\pi^{kl}_k\,\|\,\pi_{\mathrm{ref},k}\bigr)\le \delta_H,
\end{equation}
where $\epsilon_s\geq 0$ is a small tolerance that prevents noise-driven promotions and $\delta_H$ prevents aggressive shifts. If Eq. ~\eqref{eq:trust-region} holds, the inspector updates the reference for the next phase, $\pi_{\mathrm{ref},k+1} \leftarrow \pi^{kl}_{k}$; otherwise, it maintains $\pi_{\mathrm{ref},k+1} \leftarrow \pi_{\mathrm{ref},k}$. This accept/reject rule is a conservative realization of the ideal KL-regularized progression in Eq. ~\eqref{eq:kl-regularization}, made operational by restricting to $\mathcal{C}_k$ in Eq. ~\eqref{eq:practical-kl-operator} and using the $\widehat{KL}_k$.

\subsection{Theorem for Adaptive Reference Analysis}
\label{sec:theory}
\paragraph{Scope of the theory.} The regret analysis is intended to isolate the effect of adaptive reference progression under drift, not to provide an end-to-end guarantee for full sequence-level LLM fine-tuning. The theoretical model abstracts EvoDPO as a non-stationary preference contextual bandit with a KL-regularized comparator. Therefore, the theorem does not directly model the finite candidate set, the empirical token-level KL estimator, or the discrete accept/reject inspector gate used in implementation. These components are treated as practical approximations of the idealized reference-progression mechanism and are evaluated empirically through ablations, acceptance-rate statistics, and sensitivity analyses. The non-stationary contextual bandit experiment in this section is therefore best interpreted as a reward-feedback stress test of adaptive reference management, rather than a direct instantiation of the Bradley--Terry preference-bandit abstraction used in the theorem.

\paragraph{Problem Formulation}
We consider a non-stationary preference-based contextual bandit that abstracts the reference-progression mechanism used in EvoDPO. At each round t, the environment reveals a context $x_t \in \mathcal{X}$, the agent samples two actions $y^{(1)}_t, y^{(2)}_t \sim \pi_t(\cdot | x_t)$, where the action space $\mathcal{Y}$ is a finite set. Thus $\pi_t(\cdot|x_t)$ is a categorical distribution over $\mathcal{Y}$.
The preference outcome follows a Bradley-Terry model with utility
\begin{equation}
    u_t(x,y) := \langle\theta_t, \phi(x,y)\rangle
\end{equation}
\begin{equation}
    \Pr\!\left(y \succ y' \mid x,t\right) = \sigma\!\left(u_t(x,y) - u_t(x,y')\right)
\end{equation}
where $\phi(x,y)\in\mathbb{R}^d$ is a known feature map and $\theta_t$ drifts with a total variation budget $V_T$, where $V_T >0$ is a known constant to the agent.
\begin{equation}
    \sum^{T}_{t=2} \|\theta_{t} - \theta_{t-1}\|_2 \leq V_T
\end{equation}
Define the value of the policy at time $t$ as
\begin{equation}
    J_t(\pi) := \mathbb{E}_{y_t\sim \pi(\cdot \mid x_t)} [u_t(x_t,y_t)].
\end{equation}
Although EvoDPO is implemented in phases, the same process can be indexed per step $t$ by viewing the reference as piecewise constant over time. We define $\pi^*_t$ as the instantaneous optimal policy, $\pi_t$ as the policy produced by the algorithm at time $t$. We also define the KL-regularized comparator $\pi^{kl}_t$ for the reference update.
\begin{equation}
\label{eq:ideal_estimated_policy}
\pi_t := \arg\max_{\pi\in\Pi}
\left\{\widehat{J}_t(\pi) - \beta_{\mathrm{ref}} D_{\mathrm{KL}}
\bigl(\pi(\cdot|x_t)\| \pi_{\mathrm{ref},t-1}(\cdot|x_t) \bigr) \right\}
\end{equation}
Here $\widehat{J}_t$ is induced by the sliding-window preference estimator in Appendix~\ref{app:theory_proofs}; this idealized policy is used only for the regret analysis and should not be interpreted as the exact finite-sample sequence-level DPO update used in implementation.
\begin{equation}
    \pi^{*}_t := \arg \max_{\pi \in \Pi} J_t(\pi)
\end{equation}
\begin{equation}
\label{eq:kl-regularization}
    \pi^{kl}_t := \arg \max_{\pi \in \Pi} \bigl\{J_t(\pi)-\beta_{\mathrm{ref}} D_{KL}(\pi(\cdot|x_t) || \pi_{ref,t-1}(\cdot|x_t))\bigr\}
\end{equation}

We define cumulative dynamic regret as
\begin{equation}
\label{eq:cumulative_regret}
    R_T = \sum^T_{t=1} r_t.
\end{equation}

\paragraph{Dynamic Regret Analysis.}
The cumulative dynamic regret decomposes as
\[
R_T = R_T^{\mathrm{error}} + R_T^{\mathrm{bias}}.
\]
Combining the learning-error bound with the reference-induced bias bound gives, with probability at least $1-\delta$,
\begin{equation}
\label{eq:main_regret_bound}
R_T \leq \mathcal{O} \left( T^{\kappa}V_T + T^{1-\frac{\kappa}{2}}\sqrt{\log(T/\delta)} + T^{1-\kappa} + V_T
\right),
\end{equation}
where the initial reference-mismatch term is absorbed into the constant. The formal derivation is provided in Appendix~\ref{app:theory_proofs}.

\section{Numerical Results}
\label{Sec:Results}

In this section, we conduct a comprehensive set of experiments to evaluate our proposed methods. The main finding is that structured collaboration between an Evolving Agent and Role-Specialized Multi-LLM Training supporters enables the ATLAS to solve complex optimization problems that require adaptation to non-stationary environments and time-dependent PDE across varying Reynolds number. 

\subsection{Experimental Setup and Domains}

We evaluate ATLAS's capacity for long-horizon self-improvement across diverse optimization landscapes. We deploy the framework in four challenging domains. First, to test decision-making under concept drift, we evaluate the system on a non-stationary contextual bandit, where the underlying environment parameter shifts continuously under a known total variation budget. Second, to demonstrate continuous optimization capabilities in scientific machine learning (SciML), we task the agent with solving the 1D viscous Burgers' equation using a PINN. Here, the Evolving Agent must adaptively reweight a composite loss function to prevent mode collapse as a drifting viscosity schedule pushes the environment from smooth diffusion to a stiff, convection-dominated shock regime. Finally, to assess discrete combinatorial generalization, we introduce Traveling Salesperson Problem (TSP) and Bin Packing tasks, requiring the agent to design scalable heuristic routing policies. Across all domains, we benchmark ATLAS against fixed-reference evolution (EvoTune) and representative automated-discovery baselines, including FunSearch and AutoGen. Full mathematical formulations and hyperparameter configurations for all environments are detailed in Appendix \ref{baseline}.

\paragraph{Baselines.} We evaluate the performance of our full framework, ATLAS, through an ablation study and a comparison against a standard baseline (detailed in Appendix \ref{baseline}) to isolate the contributions of our proposed components:

(i) EvoTune: A baseline method that utilizes a fixed reference policy and maintains static hyperparameters over the entire course of training.

(ii) FunSearch: We implement FunSearch as an evaluator-driven code-generation baseline that iteratively proposes candidate programs, scores them using the same task evaluator, and retains high-performing programs in an experience buffer. Unlike ATLAS, FunSearch does not update the underlying Evolving Agent with DPO and does not use an adaptive reference model.

(iii) AutoGen: We implement AutoGen as a multi-agent conversation baseline in which agents propose, critique, and revise executable candidate solutions using the same task-specific evaluator. Unlike ATLAS, AutoGen performs inference-time coordination only and does not fine-tune the Evolving Agent or apply proxy-KL gate reference promotion.

\subsection{Results}

\begin{figure*}[t]
    \centering

    \begin{subfigure}[b]{0.48\textwidth}
        \centering
        \includegraphics[width=\textwidth]{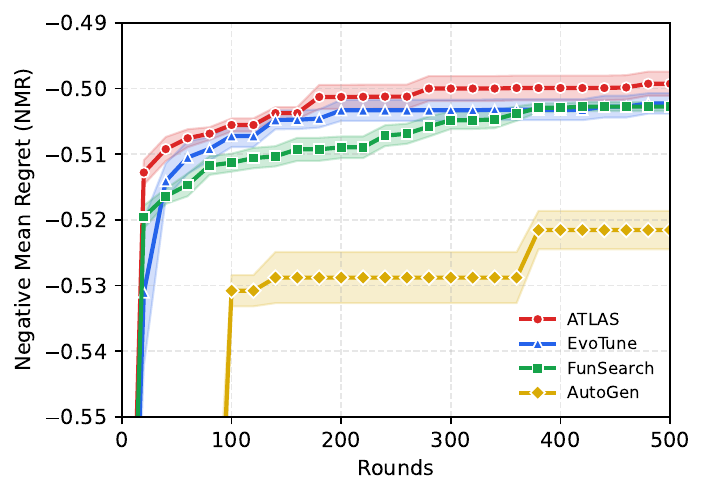}
        \caption{Bandit}
        \label{fig:bandit}
    \end{subfigure}
    \hfill
    \begin{subfigure}[b]{0.48\textwidth}
        \centering
        \includegraphics[width=\textwidth]{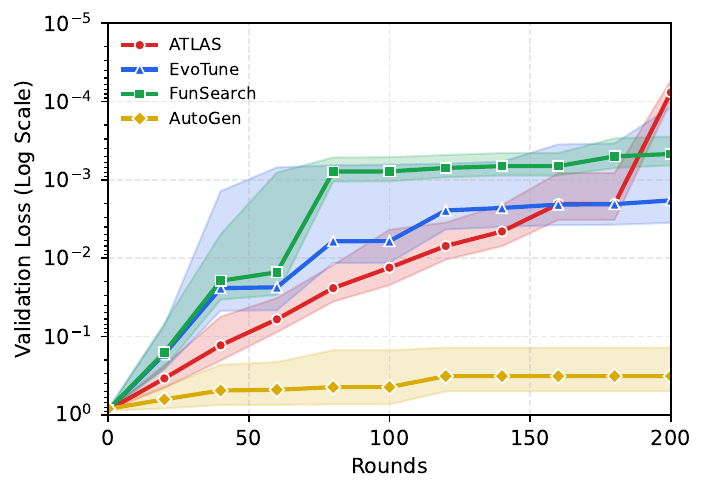}
        \caption{PINN}
        \label{fig:pinn}
    \end{subfigure}

    \vspace{0.5em}

    \begin{subfigure}[b]{0.48\textwidth}
        \centering
        \includegraphics[width=\textwidth]{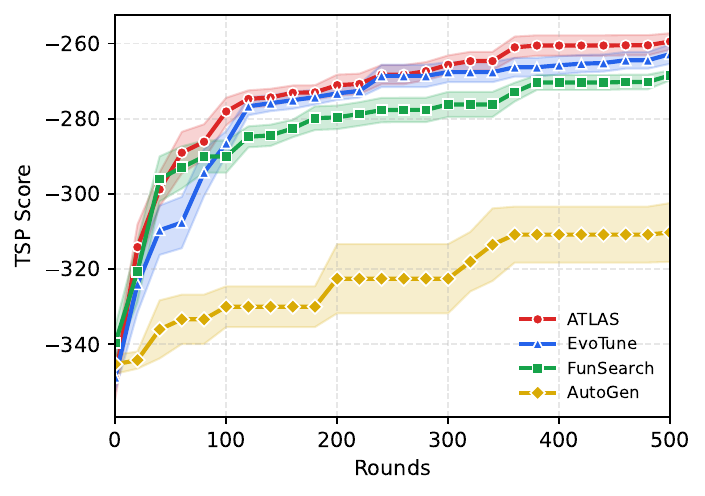}
        \caption{TSP}
        \label{fig:tsp}
    \end{subfigure}
    \hfill
    \begin{subfigure}[b]{0.48\textwidth}
        \centering
        \includegraphics[width=\textwidth]{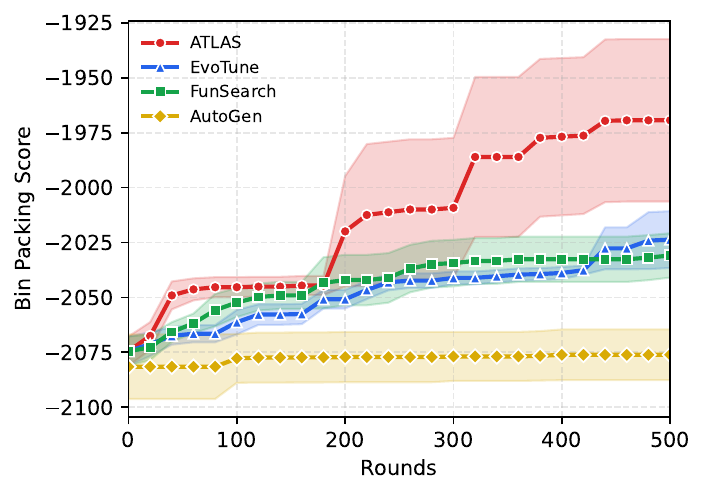}
        \caption{Bin Packing}
        \label{fig:bin}
    \end{subfigure}

    \caption{Experimental results across four executable optimization domains. 
    (a) Bandit Negative Mean Regret (NMR), shown with a truncated y-axis for readability. 
    (b) PINN fixed-evaluator validation loss on a log scale. 
    (c) TSP evaluator score. 
    (d) Bin Packing evaluator score. 
    Shaded regions denote the standard error of the mean (SEM) across 10 independent seeds.}
    \label{fig:main_results}

    \vspace{0.8em}

    \captionof{table}{Main 10-seed performance summary across four executable optimization domains. For Bandit, TSP, and Bin Packing, values report relative improvement from the initial policy to the final policy, where higher is better. For PINN, values report the final fixed-evaluator validation loss, where lower is better.}
    \label{tab:main_results}

    \small
    \begin{tabular}{lcccc}
    \toprule
    Method & Bandit Improve. $\uparrow$ & PINN Final Loss $\downarrow$ & TSP Improve. $\uparrow$ & Bin Packing Improve. $\uparrow$ \\
    \midrule
    EvoTune   & $20.23\%$ & $1.80\times 10^{-3}$ & $24.58\%$ & $2.44\%$ \\
    AutoGen   & $2.78\%$  & $3.214\times 10^{-1}$ & $10.17\%$ & $4.36\%$ \\
    FunSearch & $20.16\%$ & $4.70\times 10^{-4}$ & $20.36\%$ & $2.10\%$ \\
    ATLAS     & $\mathbf{20.71\%}$ & $\mathbf{7.70\times 10^{-5}}$ & $\mathbf{25.52\%}$ & $\mathbf{5.06\%}$ \\
    \bottomrule
    \end{tabular}
\end{figure*}

\paragraph{Experimental setup.}
We evaluate ATLAS across four executable optimization domains: non-stationary contextual bandits, PINN loss reweighting for Burgers' equation, Traveling Salesperson Problem (TSP), and Bin Packing. Figure~\ref{fig:main_results} shows the learning trajectories over self-improvement rounds, and Table~\ref{tab:main_results} summarizes final performance.

Each task is evaluated using its own task-specific metric. For Bandit, we report relative improvement in Negative Mean Regret (NMR). For PINN, we report the final fixed-evaluator validation loss. For TSP, we report relative improvement in evaluator score based on the negative optimality gap against the exact-solver reference. For Bin Packing, we report relative improvement in evaluator score based on the negative extra-bin count relative to the integer-programming reference. Full hyperparameters, hardware details, seed counts, and task configurations are provided in Appendix~\ref{appendix:param}.


\paragraph{Bandit.}
In the non-stationary contextual bandit task, ATLAS achieves the largest relative improvement in Negative Mean Regret (NMR). ATLAS improves by $20.71\%$, compared with $20.23\%$ for EvoTune, $20.16\%$ for FunSearch, and $2.78\%$ for AutoGen. Although the margin over EvoTune and FunSearch is modest, ATLAS maintains the best final trajectory in Figure~\ref{fig:main_results}(a), suggesting that supporter-guided exploration and adaptive reference management provide stable gains under parameter drift.

\paragraph{PINN.}
For the PINN task, candidate-generated losses are used only for training. All candidates are ranked and reported using a fixed external evaluator that is not generated by the agent, including normalized PDE residual, initial-condition error, boundary-condition error, and solution $L_2$ error against a high-resolution reference solution. This prevents trivial improvements caused by rescaling the generated training objective. ATLAS obtains the lowest final fixed-evaluator validation loss, $7.70\times10^{-5}$, compared with $1.80\times10^{-3}$ for EvoTune, $4.70\times10^{-4}$ for FunSearch, and $3.214\times10^{-1}$ for AutoGen. This indicates that ATLAS achieves stronger error reduction under the fixed evaluator while avoiding dependence on the generated training objective itself.

\paragraph{TSP.}
For the Traveling Salesperson Problem, ATLAS improves the evaluator score by $25.52\%$, compared with $24.58\%$ for EvoTune, $20.36\%$ for FunSearch, and $10.17\%$ for AutoGen. The TSP evaluator measures tour quality using the negative optimality gap relative to an exact-solver reference, while also enforcing validity constraints such as visiting every node exactly once and returning to the start node. These results suggest that supporter-guided structural modifications, such as route repair and local-improvement heuristics, help the Evolving Agent move beyond simple greedy tour construction.

\paragraph{Bin Packing.}
For Bin Packing, ATLAS improves the evaluator score by $5.06\%$, compared with $2.44\%$ for EvoTune, $2.10\%$ for FunSearch, and $4.36\%$ for AutoGen. The Bin Packing evaluator measures solution quality using the negative extra-bin count relative to an integer-programming reference, while checking feasibility constraints such as assigning every item exactly once and respecting bin capacities. The stronger improvement over EvoTune and FunSearch indicates that the supporter layer helps discover more effective packing heuristics, while the Policy Inspector prevents unstable reference promotions that improve isolated instances but harm feasibility or generalization.


\subsection{Ablation Study}
\label{ablation}

To isolate the contribution of each component, we evaluate the system incrementally in Table~\ref{tab:ablation_final}. Moving from EvoTune to EvoDPO isolates the effect of adaptive-reference preference optimization without the multi-LLM supporter layer. This transition substantially improves PINN performance, reducing validation loss from $1.829\times10^{-3}$ to $3.53\times10^{-4}$, improves TSP score from $-262.8154$ to $-261.7880$, and improves Bin Packing score from $-2023.6720$ to $-2004.7350$. These gains indicate that adaptive reference evolution can improve stability and late-phase adaptation in several domains even without external supporter agents.

Adding the role-specialized supporter layer yields the full ATLAS system and produces the best performance across all four tasks. ATLAS improves Bandit NMR to $-0.4993$, reduces PINN validation loss to $7.7\times10^{-5}$, improves TSP score to $-259.4793$, and improves Bin Packing score to $-1969.3510$. These results support the interpretation that EvoDPO contributes reference-adaptive training stability, while supporter agents improve candidate quality and search diversity. The combination is especially effective in combinatorial tasks, where structural heuristic modifications help the Evolving Policy move beyond narrow candidate-generation patterns.

\begin{table}[htbp]
\centering
\caption{Ablation study comparing final performance across all tasks. Bandit values report Negative Mean Regret (NMR), where higher is better. PINN values report fixed-evaluator validation loss, where lower is better. TSP and Bin Packing values report task-specific executable solver scores, where higher is better. Metrics are averaged across 10 independent seeds.}

\label{tab:ablation_final}
\small

\begin{tabular}{lcccc}
\toprule
Configuration & Bandit $\uparrow$ & PINN $\downarrow$ & TSP $\uparrow$ & Bin Packing $\uparrow$ \\

\midrule
Baseline (EvoTune)       & -0.5023 & 0.001829 & -262.8154 & -2023.6720 \\
+ Adaptivity (EvoDPO)    & -0.5023 & 0.000353 & -261.7880 & -2004.7350 \\
+ LLM Supporters (ATLAS) & \textbf{-0.4993} & \textbf{0.000077} & \textbf{-259.4793} & \textbf{-1969.3510} \\
\bottomrule
\end{tabular}
\end{table}

The ablation results show that ATLAS gains come from complementary mechanisms. EvoDPO improves stability by reducing stale-reference effects under drifting preference data, while the supporter layer improves exploration by expanding the candidate distribution with structural edits, hyperparameter suggestions, and task-specific program modifications. Thus, ATLAS should be interpreted as combining supporter-driven exploration with gated adaptive-reference stabilization: EvoDPO does not replace exploration, but makes iterative fine-tuning less stagnant once stronger candidates are available.

\section{Conclusion}
We introduced ATLAS, a role-specialized multi-LLM training framework that decomposes long-horizon executable self-improvement into candidate exploration, fine-tuning control, and reference-policy inspection. The core algorithmic component, EvoDPO, performs preference optimization with adaptive reference evolution, promoting a new reference only when it improves evaluator score while satisfying an empirical proxy-KL gate constraint. 

We analyzed an idealized non-stationary preference contextual bandit to isolate the benefit of adaptive reference progression under drift, rather than to claim an end-to-end guarantee for full sequence-level LLM fine-tuning. Empirically, we evaluated the implemented system across decision-making, scientific machine learning, and combinatorial optimization tasks. The results suggest that gated reference evolution can mitigate stale-reference stagnation while maintaining stable long-horizon improvement.

Future work will extend this framework to broader scientific discovery settings where evaluator signals are sparse, noisy, or expensive to obtain.

\bibliographystyle{plainnat}
\bibliography{reference}
\newpage
\newpage
\appendix
\onecolumn

\section{ATLAS appendix}

This appendix provides the full system-level execution procedure for ATLAS. While the main text focuses on the core EvoDPO update and the reference-promotion rule, Algorithm~\ref{alg:multi_supporters} describes how candidate generation, evaluator-based scoring, preference-pair construction, supporter feedback, and EvoDPO fine-tuning are combined into one long-horizon self-improvement loop.

\begin{table}[httb]
  \caption{LLM agents and their assigned roles.}
  \label{tab:agent-roles}
  \centering
  \begin{small}
  \renewcommand{\arraystretch}{1.05}
  \begin{tabular}{c l}
    \toprule
    \textsc{Model} & \textsc{Role / Task allocation} \\
    \midrule
    \texttt{gpt-oss-120B}    & \makecell[l]{Exploration Supporter} \\
    \addlinespace[2pt]
    \texttt{Qwen 3-32b} & \makecell[l]{Fine-Tuning Strategist} \\
    \addlinespace[2pt]
    \texttt{Llama4-latest}   & Policy Inspector \\
    \addlinespace[2pt]
    \texttt{Llama-3.2-1B}    & Evolving Agent \\
    \bottomrule
  \end{tabular}
  \end{small}
\end{table}

The six island buffers $\{\mathcal{P}_i\}_{i=0}^{5}$ maintain diverse candidate pools across parallel exploration streams. Each buffer stores prompts, generated candidate outputs, evaluator scores, and metadata such as validity, runtime status, and task-specific diagnostics. Preference pairs are constructed only after deterministic evaluator scoring: the higher-scoring candidate is used as $y^+$ and the lower-scoring candidate is used as $y^-$. This ensures that preference labels come from executable task feedback rather than from supporter-model judgments.

\begin{algorithm}[htbp]
\caption{ATLAS}
\label{alg:multi_supporters}
\begin{algorithmic}[1]
\State \textbf{Initialize:} island buffer $\{\mathcal{P}_i\}_{i=0}^{5}$, Evolving Agent policy $\pi_{\theta,0}$, reference $\pi_{\mathrm{ref},0} \leftarrow \pi_{\theta,0}$
\State \textbf{Set:} fine-tuning round $k \leftarrow 0$
\For{$t = 1, 2, \dots$}
    \For{island $i \in \{0,\dots,5\}$ in parallel}
        \State Sample prompt $x_{i,t} \sim \mathcal{X}$
        \State Candidate set $\mathcal{C}_{i,t} = \{y_{i,t}^{(j)}\}_{j=1}^{J} \sim \pi_{\theta,k}(\cdot \mid x_{i,t})$
        \For{$j = 1,\dots,J$}
            \State Evaluate $y_{i,t}^{(j)}$ to obtain score $s_{i,t}^{(j)}$ and metadata $m_{i,t}^{(j)}$
            \State $\mathcal{P}_i \leftarrow \mathcal{P}_i \cup \{(x_{i,t}, y_{i,t}^{(j)}, s_{i,t}^{(j)}, m_{i,t}^{(j)})\}$
        \EndFor
    \EndFor
    \If{fine-tuning flag = 1 \textbf{and} $k < K$}
        \State $\mathcal{D}^{\mathrm{new}}_k \leftarrow \mathrm{ConstructPairs}(\cup_i \mathcal{P}_i)$
        \State $\mathcal{D}_k \leftarrow \mathrm{Window}(\mathcal{D}_{k-1}\cup \mathcal{D}^{\mathrm{new}}_k)$
        \State \textbf{Exploration Supporter:} proposes candidate modifications for the next exploration phase
        \State \textbf{Execute one EvoDPO phase} using \(\mathcal{D}_k\), \(\pi_{\theta,k}\), and \(\pi_{\mathrm{ref},k}\)
        \State \hspace{\algorithmicindent} \textbf{Fine-Tuning Strategist:} set hyperparameters
        \State \hspace{\algorithmicindent} \textbf{Policy Inspector:} reference management
        \State $k \leftarrow k+1$
    \EndIf
\EndFor
\end{algorithmic}
\end{algorithm}
The operator $\mathrm{Window}(\cdot)$ keeps recent preference pairs and prevents stale comparisons from dominating under non-stationary task distributions. If exploration fails to produce candidates with sufficient evaluator-score separation, the fine-tuning phase can be skipped, reducing unnecessary DPO updates from weak or noisy preference pairs. The Exploration Supporter affects the next exploration phase by proposing candidate modifications, while the Fine-Tuning Strategist and Policy Inspector act inside EvoDPO: the Strategist adjusts DPO hyperparameters, and the Inspector decides whether the candidate reference should be promoted.

\section{Proofs for Section 2.3}
\label{app:theory_proofs}
This appendix provides detailed derivations for the dynamic regret bound of evolving direct preference optimization (EvoDPO). 

\subsection{Scope and Limitation of the Theoretical Abstraction}
\label{app:scope}

Our dynamic regret analysis is intentionally stylized: it studies a non-stationary linear preference contextual bandit under a Bradley--Terry model with a finite action set and linear utilities. This abstraction isolates one mechanism that is common to EvoDPO and long-horizon preference optimization---reference progression under drift---but it does not attempt to faithfully model full sequence-level LLM fine-tuning dynamics. In particular, the analysis does not capture (i) the accept/reject nature of the inspector gate, (ii) restriction to finite candidate sets $C_k$ of checkpoints, or (iii) the token-level KL proxy used in practice. These components are treated as algorithmic approximations and are validated empirically via acceptance-rate statistics and sensitivity/robustness diagnostics (App.~\ref{app:klproxy}--\ref{app:robust}).

\paragraph{Interpretation.}
The theorem should be read as evidence that systematic reference refresh can yield sublinear dynamic regret in a drifting preference environment, whereas a fixed reference can incur linear bias under drift; the gated implementation aims to approximate systematic refresh while enforcing conservative stability constraints, whose reliability we quantify empirically (App.~\ref{app:klproxy}--\ref{app:robust}).

\subsection{Preliminaries and Definitions}
\label{app:preliminaries}
We first formalize the sliding-window estimator and its requisite assumptions. Let
$W_t$ denote the sliding window of recent comparison rounds with size
$W=|W_t|=T^\kappa$, where $\kappa\in(0,1)$.

For each comparison round $\tau$, the learner observes a context $x_\tau$, two candidate actions
$y_\tau^{(1)},y_\tau^{(2)}$, and a binary preference label
\[
p_\tau=\mathbf{1}\{y_\tau^{(1)}\succ y_\tau^{(2)}\}.
\]
We define the pairwise feature difference
\[
\psi_\tau
=
\phi(x_\tau,y_\tau^{(1)})
-
\phi(x_\tau,y_\tau^{(2)}).
\]
Under the Bradley--Terry preference model,
\[
\mathbb{P}(p_\tau=1\mid x_\tau,y_\tau^{(1)},y_\tau^{(2)})
=
\sigma(\langle\theta_\tau,\psi_\tau\rangle).
\]
The sliding-window estimator is
\[
\widehat{\theta}_t
=
\arg\min_{\theta}
\sum_{\tau\in W_t}
\left[
\log(1+\exp(\theta^\top\psi_\tau))
-
p_\tau\theta^\top\psi_\tau
\right]
+
\frac{\lambda}{2}\|\theta\|_2^2.
\]
The corresponding regularized covariance matrix is
\[
A_t
=
\sum_{\tau\in W_t}
\psi_\tau\psi_\tau^\top
+
\lambda I_d.
\]
Since $\|\phi(x,y)\|_2\leq\phi_{\max}$, we have
\[
\|\psi_\tau\|_2\leq 2\phi_{\max}.
\]

\paragraph{Oracle-optimal action} Define the oracle-optimal action
\[
y^*_t(x) \in \arg \max_{y \in \mathcal{Y}} u_t(x,y).
\]

\begin{assumption}[Covariance Diversity]
\label{ass:covariance}
There exists a constant $c>0$ such that for all $t\geq W$,
\[
\lambda_{\min} \left(\sum_{\tau\in W_t} \psi_\tau\psi_\tau^\top + \lambda I_d \right) \geq cW.
\]
\end{assumption}

\begin{assumption}[Uniform logistic curvature]
\label{ass:logistic_curvature}
There exists \(m_0>0\) such that for all relevant \(z\),
\[
\sigma(z)(1-\sigma(z)) \geq m_0.
\]
\end{assumption}

We also use that $\sigma$ is $L_{\sigma}$-Lipschitz with $L_{\sigma} \leq \frac{1}{4}$:
\[
|\sigma(a) - \sigma(b)| \leq L_{\sigma} |a-b|.
\]

\begin{assumption}[Boundedness]
   We assume the feature map, parameter, and utilities are uniformly bounded,
    \[
    \|\phi(x,y)\|_2 \leq \phi_{max} , \quad \|\theta_t\|_2 \leq \theta_{max}, \quad |u_t(x,y)| \leq u_{max}.
    \] 
\end{assumption}

\begin{assumption}[Comparator support]
\label{ass:comparator_support}
We assume the KL-regularized comparator has full support, uniformly lower-bounded by \(\pi_{\min}^{{kl}}>0\):
\[
\pi_t^{{kl}}(y\mid x)\geq \pi_{\min}^{{kl}}
\]
for all relevant \(t,x,y\). This assumption can be derived from a uniformly supported reference policy and bounded utilities for the exponential-tilted comparator, but we state it directly for clarity.
\end{assumption}

\begin{assumption}[Uniform contextual margin]
\label{ass:uniform_margin}
There exists $\gamma>0$ such that for every time $t$ and every relevant context $x$, the optimal action
\[
y_t^*(x)\in\arg\max_{y\in\mathcal{Y}}u_t(x,y)
\]
is unique and satisfies
\[
u_t(x,y_t^*(x)) - \max_{y\neq y_t^*(x)}u_t(x,y) \geq \gamma.
\]
\end{assumption}

\subsection{Auxiliary Lemmas}
\begin{lemma}[Oracle Switching Budget]
\label{lemma:switching_budget}
Under Assumption~\ref{ass:uniform_margin}, the number of switches of the context-indexed oracle action satisfies
\[
B_T \leq \frac{2\phi_{max}}{\gamma} \sum^{T}_{t=2} \|\theta_t - \theta_{t-1}\|_2 \leq \frac{2\phi_{max}}{\gamma} V_T.
\]
Here
\[
B_T := \sum_{t=2}^T
\mathbf{1}\{y_t^*(x_t)\neq y_{t-1}^*(x_t)\},
\]
where both oracle actions are evaluated at the same context $x_t$. This convention is used only for the switching analysis and is justified by the uniform contextual margin assumption.
\end{lemma}

\begin{proof}
Fix $t\ge 2$ and abbreviate $y_t^*:=y_t^*(x_t)$ and $y_{t-1}^*:=y_{t-1}^*(x_t)$, the oracle-optimal action defined in \ref{app:preliminaries}.
If $y_t^*=y_{t-1}^*$, there is nothing to show. Suppose $y_t^*\neq y_{t-1}^*$. By optimality at time $t$,
\[
u_t(x_t,y_t^*) \ge u_t(x_t,y_{t-1}^*).
\]
Using linear utilities and $\|\phi(x,y)\|_2\le \phi_{\max}$, for any $y$,
\[
|u_t(x_t,y)-u_{t-1}(x_t,y)|
=|\langle \theta_t-\theta_{t-1},\phi(x_t,y)\rangle|
\le \phi_{\max}\|\theta_t-\theta_{t-1}\|_2.
\]
Therefore,
\[
u_t(x_t,y_t^*) \le u_{t-1}(x_t,y_t^*)+\phi_{\max}\|\theta_t-\theta_{t-1}\|_2,
\]
\[
u_t(x_t,y_{t-1}^*) \ge u_{t-1}(x_t,y_{t-1}^*)-\phi_{\max}\|\theta_t-\theta_{t-1}\|_2.
\]
Combining with $u_t(x_t,y_t^*) \ge u_t(x_t,y_{t-1}^*)$ gives
\[
u_{t-1}(x_t,y_t^*) - u_{t-1}(x_t,y_{t-1}^*)
\ge -2\phi_{\max}\|\theta_t-\theta_{t-1}\|_2.
\]
But by the margin condition at time $t-1$ (for context $x_t$), since $y_{t-1}^*$ is optimal under $u_{t-1}(x_t,\cdot)$ and $y_t^*\neq y_{t-1}^*$,
\[
u_{t-1}(x_t,y_{t-1}^*) - u_{t-1}(x_t,y_t^*) \ge \gamma,
\]
equivalently $u_{t-1}(x_t,y_t^*) - u_{t-1}(x_t,y_{t-1}^*)\le -\gamma$.
Thus, $\gamma \le 2\phi_{\max}\|\theta_t-\theta_{t-1}\|_2$, which upon rearranging gives
\[
\mathbf 1\{y_t^*\neq y_{t-1}^*\}\le \frac{2\phi_{\max}}{\gamma}\|\theta_t-\theta_{t-1}\|_2.
\]
Summing over $t=2,\dots,T$ yields the claim.
\end{proof}

\begin{lemma}[Local variation bound]
\label{lemma:local_variation}
Define the local window variation
\[
V_{t,W} := \sum^{t-1}_{j=t-W} \|\theta_{j+1} - \theta_j\|_2.
\]
Then,
\[
\sum_{\tau \in W_t} \|\theta_t - \theta_{\tau}\|_2 \leq W V_{t,W}, \quad \sum^{T}_{t=1}V_{t,W} \leq WV_T.
\]   
\end{lemma}
\begin{proof}
For any $\tau\in\mathcal W_t$, $\|\theta_t-\theta_\tau\|_2 \le V_{t,W}$. Therefore, $\sum_{\tau\in\mathcal W_t}\|\theta_t-\theta_\tau\|_2 \le W V_{t,W}$.

For the second claim, note that each increment $\|\theta_{j+1}-\theta_j\|_2$ appears in at most $W$ windows $\{V_{t,W}\}_{t=1}^T$.
Therefore,
\[
\sum_{t=1}^T V_{t,W}
= \sum_{t=1}^T \sum_{j=t-W}^{t-1}\|\theta_{j+1}-\theta_j\|_2
\le \sum_{j=1}^{T-1} W\,\|\theta_{j+1}-\theta_j\|_2
= W V_T.
\]
\end{proof}

\begin{lemma}[Self-normalized inequality]
\label{lemma:self_normalized}
Let
\[
\eta_\tau = p_\tau - \sigma(\langle\theta_\tau,\psi_\tau\rangle)
\]
be conditionally zero-mean and 1-sub-Gaussian. Since
$\|\psi_\tau\|_2\leq 2\phi_{\max}$, with probability at least $1-\delta$,
\[
\left\| \sum_{\tau\in W_t} \eta_\tau\psi_\tau \right\|_2 \leq \sqrt{\lambda+4W\phi_{\max}^2} \sqrt{2\left[\frac{d}{2} \log\left( 1+\frac{4W\phi_{\max}^2}{d\lambda} \right) + \log\left(\frac{1}{\delta}\right) \right]
}.
\]
\end{lemma}
\begin{proof}
We use the self-normalized vector martingale inequality following
\cite{NIPS2011_e1d5be1c}. After rewriting the preference observation in
pairwise form, the relevant feature vector is
\[
\psi_\tau
=
\phi(x_\tau,y_\tau^{(1)})
-
\phi(x_\tau,y_\tau^{(2)}),
\]
and the noise variable is
\[
\eta_\tau
=
p_\tau
-
\sigma(\langle \theta_\tau,\psi_\tau\rangle),
\]
which is conditionally zero-mean under the Bradley--Terry preference model.
The regularized design matrix is
\[
A_t
=
\sum_{\tau\in\mathcal W_t}
\psi_\tau\psi_\tau^\top
+
\lambda I_d .
\]
Since $\|\phi(x,y)\|_2\leq \phi_{\max}$, we have
\[
\|\psi_\tau\|_2
\leq
\|\phi(x_\tau,y_\tau^{(1)})\|_2
+
\|\phi(x_\tau,y_\tau^{(2)})\|_2
\leq
2\phi_{\max}.
\]

The self-normalized inequality gives, with probability at least $1-\delta$,
\[
\left\|
\sum_{\tau\in\mathcal W_t}
\eta_\tau\psi_\tau
\right\|_{A_t^{-1}}
\leq
\sqrt{
2\log\left(
\frac{
\det(A_t)^{1/2}
\det(\lambda I_d)^{-1/2}
}{\delta}
\right)
}.
\]
Now,
\[
\log\left(
\frac{
\det(A_t)^{1/2}
\det(\lambda I_d)^{-1/2}
}{\delta}
\right)
=
\frac{1}{2}
\log\left(
\frac{\det(A_t)}{\det(\lambda I_d)}
\right)
+
\log\left(\frac{1}{\delta}\right).
\]
Using the determinant-trace bound and $\|\psi_\tau\|_2\leq 2\phi_{\max}$,
\[
\det(A_t)
\leq
\left(
\lambda
+
\frac{1}{d}
\sum_{\tau\in\mathcal W_t}
\|\psi_\tau\|_2^2
\right)^d
\leq
\left(
\lambda
+
\frac{4W\phi_{\max}^2}{d}
\right)^d.
\]
Since $\det(\lambda I_d)=\lambda^d$, we obtain
\[
\frac{1}{2}
\log\left(
\frac{\det(A_t)}{\det(\lambda I_d)}
\right)
\leq
\frac{d}{2}
\log\left(
1+
\frac{4W\phi_{\max}^2}{d\lambda}
\right).
\]
Therefore, with probability at least $1-\delta$,
\[
\left\|
\sum_{\tau\in\mathcal W_t}
\eta_\tau\psi_\tau
\right\|_{A_t^{-1}}
\leq
\sqrt{
2\left[
\frac{d}{2}
\log\left(
1+
\frac{4W\phi_{\max}^2}{d\lambda}
\right)
+
\log\left(\frac{1}{\delta}\right)
\right]
}.
\]

Finally, converting from the $A_t^{-1}$ norm to the Euclidean norm gives
\[
\left\|
\sum_{\tau\in\mathcal W_t}
\eta_\tau\psi_\tau
\right\|_2
\leq
\sqrt{\lambda_{\max}(A_t)}
\left\|
\sum_{\tau\in\mathcal W_t}
\eta_\tau\psi_\tau
\right\|_{A_t^{-1}} .
\]
Using
\[
\lambda_{\max}(A_t)
\leq
\lambda+
\sum_{\tau\in\mathcal W_t}
\|\psi_\tau\|_2^2
\leq
\lambda+4W\phi_{\max}^2,
\]
we conclude that, with probability at least $1-\delta$,
\[
\left\|
\sum_{\tau\in\mathcal W_t}
\eta_\tau\psi_\tau
\right\|_2
\leq
\sqrt{\lambda+4W\phi_{\max}^2}
\sqrt{
2\left[
\frac{d}{2}
\log\left(
1+
\frac{4W\phi_{\max}^2}{d\lambda}
\right)
+
\log\left(\frac{1}{\delta}\right)
\right]
}.
\]
\end{proof}

\begin{lemma}[Parameter estimation error bound]
\label{lemma:estimate_error}
Under Assumptions~\ref{ass:covariance} and~\ref{ass:logistic_curvature}, let
\[
\eta_\tau = p_\tau-\sigma(\langle\theta_\tau,\psi_\tau\rangle)
\]
be conditionally zero-mean and 1-sub-Gaussian. Then, for a fixed time $t$, with probability at least $1-\delta$,
\[
\|\widehat{\theta}_t-\theta_t\|_2 \leq \frac{4L_\sigma\phi_{\max}^2}{m_0c}V_{t,W} + \frac{\sqrt{\lambda+4W\phi_{\max}^2}}{m_0cW}
\sqrt{ 2\left[ \frac{d}{2}
\log\left(1+\frac{4W\phi_{\max}^2}{d\lambda}\right) + \log\left(\frac{1}{\delta}\right)
\right]
} + \frac{\lambda\theta_{\max}}{m_0cW}
\]
\end{lemma}

\begin{proof}
Define the sliding-window logistic objective
\[
L_t(\theta)
=
\sum_{\tau\in\mathcal W_t}
\left[
\log(1+\exp(\theta^\top\psi_\tau))
-
p_\tau\theta^\top\psi_\tau
\right]
+
\frac{\lambda}{2}\|\theta\|_2^2 .
\]
Its gradient and Hessian are
\[
\nabla L_t(\theta)
=
\sum_{\tau\in\mathcal W_t}
\left[
\sigma(\theta^\top\psi_\tau)-p_\tau
\right]\psi_\tau
+
\lambda\theta,
\]
and
\[
\nabla^2 L_t(\theta)
=
\sum_{\tau\in\mathcal W_t}
\sigma(\theta^\top\psi_\tau)
\left(1-\sigma(\theta^\top\psi_\tau)\right)
\psi_\tau\psi_\tau^\top
+
\lambda I_d .
\]
By Assumptions~\ref{ass:covariance} and~\ref{ass:logistic_curvature},
\[
\lambda_{\min}(\nabla^2L_t(\theta))\geq m_0cW.
\]
Using strong convexity and the optimality condition $\nabla L_t(\widehat{\theta}_t)=0$,
\[
m_0cW\|\widehat{\theta}_t-\theta_t\|_2^2
\leq
\langle-\nabla L_t(\theta_t),\widehat{\theta}_t-\theta_t\rangle
\leq
\|\nabla L_t(\theta_t)\|_2\|\widehat{\theta}_t-\theta_t\|_2.
\]
Thus,
\[
\|\widehat{\theta}_t-\theta_t\|_2
\leq
\frac{1}{m_0cW}\|\nabla L_t(\theta_t)\|_2.
\]

Next, decompose the gradient at $\theta_t$:
\[
\nabla L_t(\theta_t)
=
\sum_{\tau\in\mathcal W_t}
\left[
\sigma(\theta_t^\top\psi_\tau)
-
\sigma(\theta_\tau^\top\psi_\tau)
\right]\psi_\tau
-
\sum_{\tau\in\mathcal W_t}
\eta_\tau\psi_\tau
+
\lambda\theta_t.
\]
By the Lipschitz property of $\sigma$ and $\|\psi_\tau\|_2\leq2\phi_{\max}$,
\[
\left\|
\sum_{\tau\in\mathcal W_t}
\left[
\sigma(\theta_t^\top\psi_\tau)
-
\sigma(\theta_\tau^\top\psi_\tau)
\right]\psi_\tau
\right\|_2
\leq
\sum_{\tau\in\mathcal W_t}
L_\sigma
|\langle\theta_t-\theta_\tau,\psi_\tau\rangle|
\|\psi_\tau\|_2
\]
\[
\leq
\sum_{\tau\in\mathcal W_t}
L_\sigma
\|\theta_t-\theta_\tau\|_2
\|\psi_\tau\|_2^2
\leq
4L_\sigma\phi_{\max}^2
\sum_{\tau\in\mathcal W_t}
\|\theta_t-\theta_\tau\|_2
\leq
4L_\sigma\phi_{\max}^2 W V_{t,W}.
\]
The noise term is bounded by Lemma~\ref{lemma:self_normalized}, and
$\lambda\|\theta_t\|_2\leq\lambda\theta_{\max}$. Therefore,
\[
\|\nabla L_t(\theta_t)\|_2
\leq
\]
\[
4L_\sigma\phi_{\max}^2 W V_{t,W}
+
\sqrt{\lambda+4W\phi_{\max}^2}
\sqrt{
2\left[
\frac{d}{2}
\log\left(
1+\frac{4W\phi_{\max}^2}{d\lambda}
\right)
+
\log\left(\frac{1}{\delta}\right)
\right]
}
+
\lambda\theta_{\max}.
\]
Dividing by $m_0cW$ gives the result.
\end{proof}

\begin{lemma}[KL divergence bound]
\label{lemma:kl_bound}
Let $u_t(x_t,y)=\langle\theta_t,\phi(x_t,y)\rangle$ and
$\widehat{u}_t(x_t,y)=\langle\widehat{\theta}_t,\phi(x_t,y)\rangle$.
For the idealized KL-regularized policies induced by $u_t$ and $\widehat{u}_t$ with regularization coefficient $\beta_{\mathrm{ref}}$,
\[
D_{\mathrm{KL}}\!\left(\pi^{\mathrm{kl}}_t(\cdot|x_t)\,\|\,\pi_t(\cdot|x_t)\right)
\leq
\frac{\phi_{\max}^2\|\theta_t-\widehat{\theta}_t\|_2^2}{2\beta_{\mathrm{ref}}^2}.
\]
\end{lemma}

\begin{proof}
From the definition of KL-divergence:
\[
D_{KL}\!\big(\pi^{kl}_t(\cdot|x_t)\,\|\,\pi_t(\cdot|x_t)\big)
=
\sum_{y_t \in\mathcal{Y}}
\pi^{kl}_t(y_t|x_t)\,
\log\frac{\pi^{kl}_t(y_t|x_t)}{\pi_t(y_t|x_t)}.
\]
For the idealized KL-regularized comparator analyzed in this appendix, we use the exponential-tilted form
\[
\pi_t^{{kl}}(y_t|x_t)
=
\frac{
\pi_{\mathrm{ref},t}(y_t| x_t)\exp(u_t(x_t,y_t)/\beta_{\mathrm{ref}})
}{
Z_t(x_t)
}.
\]
Similarly, the estimated policy induced by $\widehat{u}_t$ is
\[
\pi_t(y_t|x_t)
=
\frac{
\pi_{\mathrm{ref},t}(y_t|x_t)\exp(\widehat{u}_t(x_t,y_t)/\beta_{\mathrm{ref}})
}{
\widehat{Z}_t(x_t)
}, 
\]

where $Z_t(x_t) = \sum_{y_t} \pi_{ref,t} (y_t|x_t) \exp(\frac{1}{\beta_{\mathrm{ref}}} u_t(x_t,y_t))$ and $\widehat{Z}_t(x_t)$ is defined 
analogously with $\widehat{u}_t$ in place of $u_t$. This idealized form is used only for the regret analysis and abstracts away from finite-sample sequence-level DPO optimization.

From the KL-divergence definition above, we first compute the log-ratio.
\[
\log{\frac{\pi^{kl}_t(y_t|x_t)}{\pi_t(y_t|x_t)}} = \log \frac{\frac{1}{Z(x_t)} \pi_{ref,t}(y_t|x_t) \exp{\frac{1}{\beta_{\mathrm{ref}}}u_t(x_t,y_t)}}{\frac{1}{\hat{Z}(x_t)} \pi_{ref,t}(y_t|x_t) \exp{\frac{1}{\beta_{\mathrm{ref}}}\hat{u}_t(x_t,y_t)}}
\]
\[
= \log \frac{\hat{Z}(x_t)}{Z(x_t)} + \frac{1}{\beta_{\mathrm{ref}}}(u_t(x_t,y_t) - \hat{u}_t(x_t,y_t)).
\]
Let $u_t(x_t,y_t) - \hat{u}_t(x_t,y_t) = \Delta u_t$. Then,
\[
\hat{Z}(x_t) = \sum_{y_t} \pi_{ref,t} (y_t|x_t) \exp(\frac{1}{\beta_{\mathrm{ref}}} (u_t(x_t,y_t) - \Delta u_t))
\]
\[
= Z(x_t) \sum_{y_t} \pi_{ref,t} (y_t|x_t) \exp(\frac{1}{\beta_{\mathrm{ref}}} (\hat{u}_t(x_t,y_t) - u_t(x_t,y_t))),
\]
\[
\frac{\hat{Z}(x_t)}{Z(x_t)} = \sum_{y_t} \pi_{ref,t}(y_t|x_t) \exp{\frac{\hat{u}_t(x_t,y_t) - u_t(x_t,y_t)}{\beta_{\mathrm{ref}}}}.
\] 
Therefore,
\[
\log{\frac{\pi^{kl}_t(y_t|x_t)}{\pi_t(y_t|x_t)}} = \log \mathbb{E}_{y_t \sim \pi^{kl}(\cdot | x_t)}[\exp{(-\frac{\Delta u_t}{\beta_{\mathrm{ref}}})}] + \frac{1}{\beta_{\mathrm{ref}}}(\Delta u_t)
\]
Returning to the definition of KL-divergence, 
\[
D_{KL}(\pi^{kl}_t(y_t|x_t) || \pi_t(y_t|x_t)) = \mathbb{E}_{y_t \sim \pi^{kl}} \biggl[\log \frac{\pi^{kl}_t(y_t|x_t)}{\pi_t(y_t|x_t)}\biggr] = \log \mathbb{E}_{y_t \sim \pi^{kl}} \biggl[\exp{\bigl(-\frac{\Delta u_t}{\beta_{\mathrm{ref}}}\bigr)}\biggr] + \frac{1}{\beta_{\mathrm{ref}}} \mathbb{E}_{\pi^{kl}} [\Delta u_t].
\]
Let $\frac{\Delta u_t}{\beta_{\mathrm{ref}}} = X$. Then,
\[
D_{KL} = -\mathbb{E}[X] +\log \mathbb{E}_{y_t \sim \pi^{kl}} \biggl[e^{-(X-\mathbb{E}[X])}\biggr] + \mathbb{E}[X].
\]
Thus, 
\[
D_{KL}(\pi^{kl}_t(y_t|x_t) || \pi_t(y_t|x_t)) = \log \mathbb{E}_{y_t \sim \pi^{kl}}\bigl[e^{-(X-\mathbb{E}[X])}\bigr].
\]
We now bound $\Delta u_t$ in terms of the parameter estimate error. By the definition of $u_t$ in (6), 
\[
\Delta u_t(y) = \theta^\top _t \phi(x_t,y_t) - \hat{\theta}^\top_t \phi(x_t,y_t) \leq \phi_{max} \cdot \|\theta_t - \hat{\theta}_t\|_2.
\]
With the above inequality, we can define the range
\[
\frac{\Delta u_t}{\beta_{\mathrm{ref}}} \in \biggl[-\frac{\phi_{max}\|\theta_t - \hat{\theta}_t\|_2}{\beta_{\mathrm{ref}}} , \frac{\phi_{max}\|\theta_t - \hat{\theta}_t\|_2}{\beta_{\mathrm{ref}}}\biggr].
\]
Setting $a:=\frac{\phi_{max}\|\theta_t-\hat{\theta}_t\|_2}{\beta_{\mathrm{ref}}}$, we have $X\in[-a,a]$.
Since $\mathbb{E}[X]\in[-a,a]$, the centered variable satisfies
\[
\max(X-\mathbb{E}[X])-\min(X-\mathbb{E}[X])\le 2a.
\]
By Hoeffding's lemma, with $s=-1$,
\[
\log \mathbb{E}\big[e^{-(X-\mathbb{E}[X])}\big]\le \frac{(2a)^2}{8}=\frac{a^2}{2}.
\]
Therefore,
\[
D_{KL}(\pi^{kl}_t(y_t|x_t) || \pi_t(y_t|x_t)) \leq \frac{\phi_{max}^2 \cdot \|\theta_t - \hat{\theta}_t\|_2^2}{2\beta_{\mathrm{ref}}^2}.
\]
\end{proof}

\subsection{Dynamic regret decomposition}
Under a non-stationary condition, we analyze the decomposed dynamic regret
\begin{equation}
\label{Dynamic regret}
    r_t = \{\underbrace{J_t(\pi^*_t) - J_t(\pi^{kl}_t)}_{\text{Reference-induced bias}}\} + \{\underbrace{J_t(\pi^{kl}_t) - J_t(\pi_t)}_{\text{Learning Error}}\}.
\end{equation}

\paragraph{Learning Error.}
From the dynamic regret decomposition in (\ref{Dynamic regret}), the cumulative learning error $R^{error}_T$ is defined as:
\[
R^{error}_T = \sum^T_{t=1} \bigl(J_t(\pi^{kl}_t) - J_t(\pi_t)\bigr).
\]
By Lemma \ref{lemma:estimate_error}, with probability at least $1-\delta$,
\begin{equation}
\label{eq:cumulative learning error}
    R^{error}_T \leq \mathcal{O}\bigl(T^{\kappa}V_T + T^{1-\frac{\kappa}{2}}\sqrt{\log (T/\delta)} + T^{1-\kappa})
\end{equation}
where $\kappa \in (0,1)$.

To apply the estimation bound uniformly over all rounds, we invoke
Lemma~\ref{lemma:estimate_error} with confidence level $\delta/T$ and take a union bound over
$t=1,\ldots,T$. Therefore, with probability at least $1-\delta$, for all $t$ simultaneously,
\[
\|\widehat{\theta}_t-\theta_t\|_2
\leq
\frac{4L_\sigma\phi_{\max}^2}{m_0c}V_{t,W}
+
\frac{\sqrt{\lambda+4W\phi_{\max}^2}}{m_0cW}
\sqrt{
2\left[
\frac{d}{2}
\log\left(1+\frac{4W\phi_{\max}^2}{d\lambda}\right)
+
\log\left(T/\delta\right)
\right]
}
+
\frac{\lambda\theta_{\max}}{m_0cW}.
\]

\paragraph{Proof of Learning Error.}
We now establish the bound on $R^{error}_T$. We have
\[
r_t = J_t(\pi^{kl}_t) - J_t(\pi_t) = \sum_{y_t \in \mathcal{Y}} \pi^{kl}_t(y_t|x_t)\;u_t(x_t,y_t) - \sum_{y_t \in \mathcal{Y}} \pi_t(y_t|x_t)\;u_t(x_t,y_t).
\]
Thus, we estimate
\[
\sum_{y_t \in \mathcal{Y}} \pi^{kl}_t(y_t|x_t)\;u_t(x_t,y_t) - \sum_{y_t \in \mathcal{Y}} \pi_t(y_t|x_t)\;u_t(x_t,y_t) \leq \sum_{y_t \in \mathcal{Y}} \biggl|\bigl(\pi^{kl}_t(y_t|x_t) - \pi_t(y_t|x_t)\bigr)\cdot u_t(x_t,y_t)\biggr|.
\]
\[
\sum_{y_t \in \mathcal{Y}} \biggl|\bigl(\pi^{kl}_t(y_t|x_t) - \pi_t(y_t|x_t)\bigr)\cdot u_t(x_t,y_t)\biggr| \leq \sum_{y_t \in \mathcal{Y}} |\pi^{kl}_t(y_t|x_t) - \pi_t(y_t|x_t)| \cdot |u_t(x_t,y_t)|.
\]
Since $\|u_t(x_t,\cdot)\|_\infty := \max_{y_t \in \mathcal{Y}} |u_t(x_t, y_t)|$, we get the following inequality
\[
\leq \|u_t(x_t,\cdot)\|_\infty \sum_{y_t\in\mathcal Y}\Bigl|\pi_t^{kl}(y_t|x_t)-\pi_t(y_t|x_t)\Bigr|.
\]

Since $ \|u\|_\infty \leq \|u\|_2 \leq \|u\|_1$ for any finite vector $u$, it follows that,
\[
r_t \leq \sum_{y_t \in \mathcal{Y}} |\pi^{kl}_t(y_t|x_t) - \pi_t(y_t|x_t)| \cdot \|u_t(x_t,\cdot)\|_\infty \leq \sum_{y_t \in \mathcal{Y}} |\pi^{kl}_t(y_t|x_t) - \pi_t(y_t|x_t)| \cdot \|u_t(x_t,\cdot)\|_2.
\]
Writing $\sum_{y_t \in \mathcal{Y}} |\pi^{kl}_t(y_t|x_t) - \pi_t(y_t|x_t)|$ as the $L_1$ norm $\|\pi^{kl}_t(y_t|x_t) - \pi_t(y_t|x_t)\|_1$ and applying Assumption B.3, we obtain the upper bound
\[
r_t \leq \theta_{max}\cdot \phi_{max} \cdot \|\pi^{kl}_t(y_t|x_t) - \pi_t(y_t|x_t)\|_1.
\]

Applying Pinsker's inequality to the $L_1$ norm and invoking Lemma \ref{lemma:kl_bound},
\[
J_t(\pi^{kl}_t) - J_t(\pi_t) \leq \sqrt{2D_{KL}(\pi^{kl}_t(y_t|x_t) || \pi_t(y_t|x_t))} \leq \frac{\theta_{max}\phi^2_{max}}{\beta_{\mathrm{ref}}}  \cdot \|\theta_t - \hat{\theta}_t\|_2.
\]
Summing over $t$ and applying Lemma \ref{lemma:estimate_error}, together with $\sum^{T}_{t=1} V_{t,W}\leq W\;V_T$,
\[
R^{error}_T \leq \sum^T_{t=1} \frac{\theta_{max}\phi^2_{max}}{\beta_{\mathrm{ref}}}  \cdot \|\theta_t - \hat{\theta}_t\|_2.
\]
Using Lemma~\ref{lemma:estimate_error} uniformly over $t=1,\ldots,T$,
\[
R_T^{\mathrm{error}}
\leq
\frac{\theta_{\max}\phi_{\max}^2}{\beta_{\mathrm{ref}}}
\sum_{t=1}^T
\|\widehat{\theta}_t-\theta_t\|_2.
\]
Therefore,
\[
R_T^{\mathrm{error}}
\leq
C
\sum_{t=1}^T
\left[
V_{t,W}
+
W^{-1/2}\sqrt{\log(T/\delta)}
+
W^{-1}
\right],
\]
where $C$ is a constant that depends on
$\theta_{\max},\phi_{\max},L_\sigma,m_0,c,\lambda,d,$ and $\beta_{\mathrm{ref}}$.
By Lemma~\ref{lemma:local_variation},
\[
\sum_{t=1}^T V_{t,W}\leq WV_T.
\]
Hence,
\[
R_T^{\mathrm{error}}
\leq
C
\left[
WV_T
+
TW^{-1/2}\sqrt{\log(T/\delta)}
+
TW^{-1}
\right].
\]
Setting $W=T^\kappa$ gives
\[
R_T^{\mathrm{error}}
\leq
\mathcal{O}
\left(
T^\kappa V_T
+
T^{1-\kappa/2}\sqrt{\log(T/\delta)}
+
T^{1-\kappa}
\right).
\]

\paragraph{Reference-Induced Bias.}
Define the reference-induced bias as
\[
R^{bias}_T = \sum^T_{t=1} \bigl(J_t(\pi^*_t) - J_t(\pi^{kl}_t) \bigr).
\]
where $\pi_t^{\mathrm{kl}}$ is the idealized KL-regularized comparator. Under Assumption~\ref{ass:uniform_margin} and the full-support condition on the reference policy, the cumulative reference-induced bias satisfies
\[
R_T^{\mathrm{bias}} \leq \mathcal{O}(1+V_T).
\]
When the initial mismatch term is bounded, the drift-dependent part is $\mathcal{O}(V_T)$.

\paragraph{Proof of Reference Bias.}
Recall the definition of $\pi^{kl}(y_t|x_t)$,
\[
\pi^{kl}_t(y_t|x_t) = \frac{1}{Z(x)} \pi_{ref,t}(y_t|x_t) \exp{\frac{1}{\beta_{\mathrm{ref}}}u_t(x_t,y_t)},
\]
\[
\log\pi^{kl}_t(y_t|x_t) = \log \pi_{ref,t} (y_t|x_t) + \frac{1}{\beta_{\mathrm{ref}}} u_t(x_t,y_t) - \log Z(x_t).
\]
Rearranging terms we obtain
\[
\frac{1}{\beta_{\mathrm{ref}}} u_t(x_t,y_t) = \log \frac{\pi^{kl}_t(y_t|x_t)}{\pi_{ref,t}(y_t|x_t)} + \log Z(x_t).
\]
Recalling the definition of $J_t(\pi)$ in (9),
\[
\frac{1}{\beta_{\mathrm{ref}}} J_t(\pi) = \sum_{y_t} \pi(y_t|x_t)\frac{u_t(x_t,y_t)}{\beta_{\mathrm{ref}}},
\]
and substituting the expression for $\frac{1}{\beta_{\mathrm{ref}}} u_t(x_t,y_t)$ we get
\[
\frac{1}{\beta_{\mathrm{ref}}} J_t(\pi) = \sum_{y_t} \pi(y_t|x_t) \biggl[\log \frac{\pi^{kl}_t(y_t|x_t)}{\pi_{ref,t}(y_t|x_t)} + \log Z(x_t) \biggr].
\]
To write the above equation in terms of KL-divergence, we decompose the log-ratio as follows:
\[
\log \frac{\pi^{kl}_t(y_t|x_t)}{\pi_{ref,t}(y_t|x_t)} = \log \frac{\pi(y_t|x_t)}{\pi_{ref,t}(y_t|x_t)} - \log \frac{\pi(y_t|x_t)}{\pi^{kl}_t(y_t|x_t)}.
\]
Substituting into the expression for $\frac{1}{\beta_{\mathrm{ref}}} J_t(\pi)$,
\[
\frac{1}{\beta_{\mathrm{ref}}} J_t(\pi) = \sum_{y_t} \pi(y_t|x_t)\biggl[\log \frac{\pi(y_t|x_t)}{\pi_{ref,t}(y_t|x_t)} - \log \frac{\pi(y_t|x_t)}{\pi^{kl}_t(y_t|x_t)} + \log Z(x_t)\biggr],
\]
\[
\frac{1}{\beta_{\mathrm{ref}}} J_t(\pi) = D_{KL}(\pi||\pi_{ref,t}) - D_{KL}(\pi||\pi^{kl}_t) + \sum_{y_t} \pi(y_t|x_t) \log Z(x_t).
\]
Therefore,
\[
J_t(\pi) = \beta_{\mathrm{ref}} \bigl( D_{KL}(\pi||\pi_{ref,t}) - D_{KL}(\pi||\pi^{kl}_t) + \sum_{y_t} \pi(y_t|x_t) \log Z(x_t)\bigr).
\]
Now, we substitute $\pi = \pi^{kl}_t$ and $\pi = \pi^*_t$. Then,
\[
J_t(\pi^*_t) - J_t(\pi^{kl}_t) = \beta_{\mathrm{ref}} \bigl(D_{KL}(\pi^*_t||\pi_{ref,t}) - D_{KL}(\pi^*_t||\pi^{kl}_t) - D_{KL}(\pi^{kl}_t||\pi_{ref,t})\bigr). 
\]
Our framework uses EvoDPO, which updates the current reference policy to the previous KL-regularized optimal policy (only for EvoDPO; if it is DPO, $\pi_{ref,t} = \pi_{ref}$). Thus, $\pi_{ref,t} = \pi^{kl}_{t-1}$. Since $D_{KL}(\pi^{kl}_t||\pi^{kl}_{t-1}) \geq 0$, dropping it yields the upper bound,
\[
J_t(\pi^*_t) - J_t(\pi^{kl}_t) \leq \beta_{\mathrm{ref}} \bigl(D_{KL}(\pi^*_t||\pi^{kl}_{t-1}) - D_{KL}(\pi^*_t||\pi^{kl}_t) \bigr).
\]

Summing over $t=1,\ldots,T$ gives
\[
R_T^{\mathrm{bias}}
\leq
\beta_{\mathrm{ref}}
\sum_{t=1}^T
\left[
D_{\mathrm{KL}}(\pi_t^*\|\pi_{t-1}^{\mathrm{kl}})
-
D_{\mathrm{KL}}(\pi_t^*\|\pi_t^{\mathrm{kl}})
\right].
\]
Because the oracle policy $\pi_t^*$ changes over time, this expression does not telescope directly. We therefore add and subtract
$D_{\mathrm{KL}}(\pi_{t-1}^*\|\pi_{t-1}^{\mathrm{kl}})$ for $t\geq 2$:
\[
R_T^{\mathrm{bias}}
\leq
\beta_{\mathrm{ref}} D_{\mathrm{KL}}(\pi_1^*\|\pi_0^{\mathrm{kl}})
+
\beta_{\mathrm{ref}}
\sum_{t=2}^T
\left[
D_{\mathrm{KL}}(\pi_t^*\|\pi_{t-1}^{\mathrm{kl}})
-
D_{\mathrm{KL}}(\pi_{t-1}^*\|\pi_{t-1}^{\mathrm{kl}})
\right].
\]
For deterministic oracle policies,
\[
D_{\mathrm{KL}}(\pi_t^*\|\pi_{t-1}^{\mathrm{kl}})
=
\log\frac{1}{\pi_{t-1}^{\mathrm{kl}}(y_t^*(x_t)\mid x_t)}.
\]
Therefore,
\[
D_{\mathrm{KL}}(\pi_t^*\|\pi_{t-1}^{\mathrm{kl}})
-
D_{\mathrm{KL}}(\pi_{t-1}^*\|\pi_{t-1}^{\mathrm{kl}})
=
\log
\frac{
\pi_{t-1}^{\mathrm{kl}}(y_{t-1}^*(x_t)\mid x_t)
}{
\pi_{t-1}^{\mathrm{kl}}(y_t^*(x_t)\mid x_t)
}.
\]
If $y_t^*(x_t)=y_{t-1}^*(x_t)$, this term is zero. If a switch occurs, the full-support assumption gives
\[
\log
\frac{
\pi_{t-1}^{\mathrm{kl}}(y_{t-1}^*(x_t)\mid x_t)
}{
\pi_{t-1}^{\mathrm{kl}}(y_t^*(x_t)\mid x_t)
}
\leq
\log\frac{1}{\pi_{\min}}.
\]
Thus, the cumulative switching contribution is bounded by
\[
\sum_{t=2}^T
\left[
D_{\mathrm{KL}}(\pi_t^*\|\pi_{t-1}^{\mathrm{kl}})
-
D_{\mathrm{KL}}(\pi_{t-1}^*\|\pi_{t-1}^{\mathrm{kl}})
\right]
\leq
B_T\log\frac{1}{\pi_{\min}}
+
\frac{\phi_{\max}}{\beta_{\mathrm{ref}}}V_T.
\]
The additional term $\frac{\phi_{\max}}{\beta_{\mathrm{ref}}}V_T$ accounts for the utility drift in the exponential-tilted policy between consecutive environments. Combining the switching and drift contributions yields
\[
R_T^{\mathrm{bias}}
\leq \beta_{\mathrm{ref}} D_{\mathrm{KL}}(\pi_1^*\|\pi_0^{\mathrm{kl}}) + \beta_{\mathrm{ref}} B_T\log\frac{1}{\pi_{\min}} + \phi_{\max}V_T.
\]
Using Lemma~\ref{lemma:switching_budget},
\[
B_T \leq \frac{2\phi_{\max}}{\gamma}V_T,
\]
we obtain
\[
R_T^{\mathrm{bias}} \leq \beta_{\mathrm{ref}} D_{\mathrm{KL}}(\pi_1^*\|\pi_0^{\mathrm{kl}})
+ \frac{2\beta_{\mathrm{ref}}\phi_{\max}}{\gamma} V_T \log\frac{1}{\pi_{\min}} + \phi_{\max}V_T.
\]
If the initial KL term is bounded by a constant, then
\[
R_T^{\mathrm{bias}} = \mathcal{O}(1+V_T).
\]
When focusing on the drift-dependent term, this is written as
\[
R_T^{\mathrm{bias}} = \mathcal{O}(V_T).
\]

\paragraph{Final bound.} Combining the learning error with the reference-induced bias gives, with probability at least $1-\delta$,
\[
R_T = R_T^{\mathrm{error}} + R_T^{\mathrm{bias}} \leq \mathcal{O}
\left(T^{\kappa}V_T + T^{1-\frac{\kappa}{2}}\sqrt{\log(T/\delta)}
+ T^{1-\kappa} + V_T + 1 \right).
\]
If the initial mismatch term is absorbed into the constant, this is written as
\[
R_T \leq \mathcal{O} \left( T^{\kappa}V_T + T^{1-\frac{\kappa}{2}}\sqrt{\log(T/\delta)} + T^{1-\kappa} + V_T \right).
\]
In particular, the average dynamic regret vanishes when
\[
V_T=o(T^{1-\kappa}),
\]
so the result is most meaningful in slowly drifting environments.

\section{Additional Information of Evaluation Setup for Section 3}
\subsection{Detailed Method Used Formulations}
\label{baseline}
To validate the effectiveness of ATLAS, we compare our proposed method against a standard single-agent baseline:
FunSearch (Baseline):
AutoGen (Baseline):

EvoTune (Baseline):
This configuration represents the standard Evolving Agent without external supporter agent guidance. The EvoTune generates solutions via standard DPO (mode=vanilla) on the self-generated data. Crucially, EVOTUNE\_USE\_TEACHER=0 is set, meaning the Exploration Supporter is disabled (no static analysis feedback) and the Strategist is disabled (use\_dpo\_strategy\_teacher=0), forcing the agent to rely solely on the scalar scores from the evaluator.

EvoDPO (Ablation): This isolates the impact of our proposed Evolving DPO (EvoDPO) algorithm (mode=adaptive) without the full agentic support of ATLAS. Unlike EvoTune (Baseline), the reference model $\pi_\text{ref,t}$ is not static; it updates to track the Evolving Agent $\pi_t$. However, unlike ATLAS, the Policy Inspector is disabled (or set to always accept), and the Supporters are disabled (EVOTUNE\_USE\_TEACHER=0), meaning the reference model updates without proxy-KL gate safeguard and qualitative feedback. This serves as an ablation to test whether the "moving reference" mechanism alone is sufficient without the Policy Inspector gating.

ATLAS (Ours):
Our full proposed framework. It utilizes the complete task-distributed multi-agent system (Exploration Supporter, Fine-Tuning Strategist ,and Policy Inspector) and distinguishes itself by replacing standard DPO with our Evolving DPO algorithm (mode=adaptive). This enables the reference model $\pi_\text{ref,t}$ to evolve over time. Crucially, every update is gated by the Policy Inspector, which rejects updates that violate the drift constraint ($\widehat{\mathrm{KL}}_k \le \delta_H$) or fail to improve validation performance.

\subsection{Detailed Problem Formulations and Evaluation Metrics}
\label{problem}
\subsubsection{Non-Stationary Contextual Bandit} We consider a bandit with $k$ arms and context dimension $d$ over a horizon of $H$ steps. At each step $t$, the decision-maker observes context vectors ${(x_{t,a})}_{a=1}^K \in \mathbb{R}^d$ (normalized to unit norm) and selects an arm $a_t$. The reward is $r_t = \langle \theta_t, x_{t,a_t} \rangle + \eta_t$, with 1-sub gaussian noise $\eta_t$. The parameter $\theta_t$ drifts with a total variation budget $V_T$, evolving as: 
\[
\theta_{t+1} = \frac{\theta_t + \delta_t}{||\theta_t + \delta_t||_2}
\]
\textbf{Evaluation.} The Evolving Agent simulates a bandit problem for each candidate policy. The system evaluates performance using the Negative Mean Regret (NMR) over the horizon $H$. This scoring metric aligns with a maximization objective (higher is better):
\[
\text{score} = -\frac{1}{H}\sum_{t=1}^H (r^*_t - r_t)
\]
where $r^*_t$ is the expected reward of the optimal arm and $r_t$ is the expected reward of the selected arm.

\subsubsection{1D Burgers' Equation} We solve the viscous Burgers' equation on the domain $(x,t) \in [-1, 1] \times [0, 1]$:
\[
u_t + uu_x - \nu(t) u_{xx} = 0
\]
with initial condition $u(0, x) = -\sin(\pi x)$ and zero Dirichlet boundary conditions $u(t, -1) = u(t, 1) = 0$.
The Physics-Informed Neural Network (PINN) $u_\theta(x,t)$ is trained to minimize a composite loss:
\[
\mathcal{L}(\theta) = \lambda_{pde}\mathcal{L}_{pde}(\theta) + \lambda_{ic}\mathcal{L}_{ic}(\theta) + \lambda_{bc}\mathcal{L}_{bc}(\theta)
\]
where $\mathcal{L}_{pde}$, $\mathcal{L}_{ic}$, and $\mathcal{L}_{bc}$ are the mean squared errors (MSE) of the PDE residuals, initial condition loss, and boundary condition loss, respectively. The Evolving Agent's task is to distribute the weighting coefficients $\lambda_{pde}, \lambda_{ic}, \lambda_{bc}$ to balance these competing objectives.

\paragraph{Physical Interpretation of Parameter Drift.}
A key challenge in PINN training is the dependency on the Reynolds number $Re \propto 1/\nu$. Low Reynolds numbers (high $\nu$) correspond to diffusion-dominated regimes where solutions are smooth and easy to train. High Reynolds numbers (low $\nu$) correspond to convection-dominated regimes where sharp gradients (shocks) form, creating a ``stiff'' optimization landscape that often leads to failure modes like the trivial zero solution or spectral bias.
To evaluate the Evolving Agent's ability to adapt to these changing physical regimes, we implement a drifting viscosity schedule:
\[
\nu(t) = \nu_{base} \cdot \left(1 + \alpha \sin\left(\frac{2\pi t}{T}\right)\right)
\]
where $\nu_{base} \approx 0.00318$ ($Re \approx 600$), $\alpha=0.9$. This drifts the system between a viscous regime ($\nu \approx 0.006, Re \approx 300$) and a highly turbulent-like shock regime ($\nu \approx 0.0003, Re \approx 6600$). This forces the Evolving Agent to design a loss strategy that is robust not just to a single static problem, but to the entire transition from smooth flow to shock formation.

\textbf{Evaluation.} We instantiate a fresh neural network (MLP with 4 hidden layers of 32 units) for each candidate loss strategy and train it for 500 optimization steps using the Adam optimizer. We then evaluate the performance on a newly sampled, held-out set of $N_{coll}$ collocation points. We report the Negative Validation Loss (NVL) as the score:
\[
\text{score} = - \left| \mathcal{L}(\theta; \mathcal{D}_{val}) \right|
\]
where $\mathcal{L}$ is the candidate's loss function evaluated on the validation set. Scores closer to 0 indicate that the resulting PINN solution satisfies the agent-defined time-dependent physical dynamic setting.

\textit{Note: To adhere to standard visualization conventions, the main body of this paper reports the magnitude of this score (i.e., $|\text{score}|$) as ``Validation Loss,'' where lower is better.}

\subsubsection{Traveling Salesperson Problem (TSP)}

We consider Euclidean TSP instances defined on a complete undirected graph
$G=(V,E)$, where $V=\{1,\ldots,n\}$ denotes the node set and each edge
$(i,j)\in E$ has distance $d_{ij}$. A candidate solution is represented by a
permutation $\pi$ of the nodes. The corresponding tour length is
\begin{equation}
\label{eq:tsp_objective}
L(\pi)
=
\sum_{i=1}^{n-1}
d_{\pi(i),\pi(i+1)}
+
d_{\pi(n),\pi(1)}.
\end{equation}
The objective is to minimize $L(\pi)$ subject to the constraint that every node
is visited exactly once and the tour returns to the starting node.

The Evolving Agent's task is to generate a valid tour in a machine-readable
JSON format. For each candidate tour, a deterministic verifier checks whether
the output forms a valid Hamiltonian cycle and computes the tour length using
the provided distance matrix. Invalid tours, including tours with duplicated or
omitted nodes, are penalized by the evaluator and are not treated as feasible
improvements.

\textbf{Evaluation.} For each TSP instance, we compare the candidate tour length
$L_{\mathrm{alg}}$ against an exact-solver reference $L_{\mathrm{opt}}$. We use
the negative optimality gap as the task score:
\begin{equation}
\label{eq:tsp_score}
\mathrm{score}
=
-
\frac{L_{\mathrm{alg}}-L_{\mathrm{opt}}}{L_{\mathrm{opt}}}
\times 100.
\end{equation}
This converts the minimization objective into a maximization score, where values
closer to $0$ indicate better tours. In the main text, we report the relative
improvement from the initial policy to the final policy. We additionally track
the valid tour rate and JSON format-compliance rate to ensure that performance
improvements are not caused by invalid or unparsable outputs.

\subsubsection{Bin Packing Problem (BPP)}

We consider one-dimensional Bin Packing instances with item weights
$\{w_i\}_{i=1}^n$, where $w_i\in(0,C]$, and identical bins of capacity $C$. A
candidate solution partitions the items into bins
$\{B_1,\ldots,B_k\}$ such that
\begin{equation}
\label{eq:bpp_constraint}
\sum_{i\in B_j} w_i \leq C,
\qquad
j=1,\ldots,k.
\end{equation}
The objective is to minimize the number of bins:
\begin{equation}
\label{eq:bpp_objective}
\min k.
\end{equation}

The Evolving Agent's task is to generate a feasible bin assignment in JSON
format. For each candidate packing, a deterministic verifier checks that every
item is assigned exactly once, no item is duplicated, and no bin exceeds the
capacity constraint. Candidate outputs with capacity violations, duplicated
items, or unassigned items are penalized by the evaluator.

\textbf{Evaluation.} Let $k_{\mathrm{alg}}$ denote the number of bins used by a
candidate solution and let $k_{\mathrm{opt}}$ denote the optimal number of bins
computed by an integer-programming solver. We define the score as the negative
extra-bin count:
\begin{equation}
\label{eq:bpp_score}
\mathrm{score}
=
-
\left(k_{\mathrm{alg}}-k_{\mathrm{opt}}\right).
\end{equation}
This again converts the minimization objective into a maximization score, where
larger values indicate fewer extra bins relative to the exact reference. In the
main text, we report the relative improvement from the initial policy to the
final policy. We additionally report constraint violation rate and packing
density as auxiliary diagnostics.

\subsection{Detailed Role Definition}
\label{role}
\subsubsection{Non-Stationary Contextual Bandit}

\textbf{Evolving Agent.}

The Evolving Agent is initialized with Llama-3.2-1B-Instruct and receives a system prompt defining it as an \textit{Evolving Agent in online learning and contextual bandits.} (see the adaptation details in Appendix \ref{prob_bandit:student}).  For each round, it is provided with the problem context and instructed to implement a Python function \texttt{policy(context, history, t, ...)} that selects an action index (to select arm) $a_t \in [0, K-1]$ which will be evaluated by the generated policy and returns the score.

\textbf{Role-Specialized Multi-LLM Training Supporters.}

The Exploration Supporter (gpt-oss:120b) analyzes the generated policy code. Specifically, the Sliding-Window LinUCB implementation dynamically suggests hyperparameter adjustments. For instance, while the initial window size is set to 50, the supporter may suggest increasing it (e.g./ to 100 or 200) during stable phases to capture longer-term dependencies, or reducing it to adapt more rapidly during high-drift phases. Similarly, it tunes the regularization parameter (e.g., from 1.0 to 0.1) and exploration scalar based on recent performance. In one instance, it flagged a policy with a large fixed window: ``The default sliding-window size ($\approx 50$) is too large for a rapidly drifting environment, so the ridge estimator incorporates stale data.'' It explicitly patched the code to: \texttt{if t == 0: policy.window\_size = 20 policy.lambda\_reg = 0.1} This reduced the estimator's "memory" to track the random walk $\theta_t$ more effectively. 

The Fine-Tuning Strategist (Qwen 3-32b) monitors training diagnostics during the periodic finetuning updates (every 20 rounds) to adjust DPO hyperparameters (score threshold, $\beta_{DPO}$, and epochs) to prevent the model from overfitting to stale bandit feedback. When it was observed that the student was generating repetitive code (low entropy), it issued a JSON command to set the DPO beta ($\beta_{DPO}$) to $0.1$ and adjust the score\_threshold to include a broader range of "good enough" examples, forcing the model to explore alternative ridge hyperparameters.

The Policy Inspector (llama4:latest) acts as a safety gate for EvoDPO, permitting reference policy updates only when the candidate $\pi^{kl}_t$ shows sufficient score improvement ($\Delta_t \geq -0.0007$) and satisfies a trust-region KL divergence constraint ($\widehat{\mathrm{KL}}_k(\pi^{kl}_k || \pi_{\text{ref}, k}) \leq 0.002$). The logs show it rejecting a candidate update that improved training regret but had a high KL divergence of $0.0045$ on the generated tokens, preventing a "catastrophic update" where the agent would have dissolved its stable baseline policy in favor of a high-variance experimental one.

\subsubsection{Physics-Informed Neural Network (Burgers' Equation)}

\textbf{Evolving Agent.}

The Evolving Agent is initialized with Llama-3.2-1B-Instruct and receives a system prompt defining it as an \textit{Evolving Agent in PDE-constrained optimization.} It is tasked with implementing a Python function \texttt{pinn\_loss\_strategy(residuals, ic\_err, bc\_err, extras)} that returns a scalar loss tensor. Note that the agent does not solve the optimization problem itself; instead, it defines the scalar objective function that the external optimizer (Adam) minimizes. The prompt explicitly encourages the agent to use adaptive weighting (based on \texttt{extras['epoch']}) and to address numerical stiffness (e.g., by re-weighting or normalizing residuals near shocks).

\textbf{Role-Specialized Multi-LLM Training Supporters.}

The Exploration Supporter (gpt-oss:120b) critiques the generated loss strategy code, identifying issues like scale imbalance between error terms and poor schedule design. For example, it detected that the PDE residual term ($\sim 10^{-2}$) often overshadowed boundary condition errors ($\sim 10^{-5}$), causing "boundary drift." In the logs, it explicitly recommended dynamic normalization and component-wise clipping (Huber loss) to fix this: \textit{``Replace global clamp with a Huber loss to prevent exploding gradients... Keep a running estimate of the standard deviation of each component and divide by it.''} This ensured the optimization respected physical boundaries at $x=\pm 1$.

The Fine-Tuning Strategist (Qwen 3-32b) dynamically tunes DPO hyperparameters during the periodic fine-tuning updates (every 20 rounds) to ensure the generator learns robust loss structures without collapsing to degenerate solutions. The logs show it actively adjusting $\beta_{DPO}$ (e.g., setting beta=0.7 then 0.5). While it occasionally suggests changes to num\_train\_epochs or score\_threshold (e.g., "Increasing beta and epochs for better exploration"), the system prioritizes its $\beta_{DPO}$ adjustments to stabilize convergence without inducing mode collapse.

The Policy Inspector (llama4:latest) enforced the improvement threshold. In one instance, it rejected an update where the candidate loss strategy minimized the PDE residual to $10^{-6}$ but caused the Initial Condition error to spike to $10^{-1}$ (a common mode collapse in PINNs known as the "trivial solution" where $u(t,x)=0$). The inspector flagged this as ``Regression: IC error spike'' and reverted the reference model, saving the training run from divergence.

\subsubsection{Traveling Salesperson Problem (TSP)}

\textbf{Evolving Agent.}

The Evolving Agent is initialized with Llama-3.2-1B-Instruct and receives a system prompt defining it as an \textit{Evolving Agent for combinatorial optimization}. For each TSP instance, the policy receives a node set and a distance matrix, and it generates a structured JSON output containing a candidate tour. The generated tour is then passed to a deterministic verifier, which checks whether the tour visits every node exactly once and returns to the starting node. The verifier also computes the tour length and the optimality gap relative to the exact-solver reference. Thus, the Evolving Agent is not rewarded for self-reported reasoning quality; it is scored only through executable validity and objective value.

\textbf{Role-Specialized Multi-LLM Training Supporters.}

The Exploration Supporter (gpt-oss:120b) analyzes generated TSP heuristics and proposes structural improvements to the candidate-generation strategy. For example, when the Evolving Agent repeatedly produced nearest-neighbor-style tours, the supporter encouraged additional local-improvement operations such as 2-opt swaps, edge-crossing removal, and route repair for duplicated or omitted nodes. These patches expand the candidate distribution beyond purely greedy construction and help prevent the policy from converging to a narrow heuristic family.

The Fine-Tuning Strategist (Qwen 3-32b) monitors the distribution of verifier scores and adjusts the DPO fine-tuning hyperparameters, including $\beta_{\mathrm{DPO}}$, score-margin thresholds, and training epochs. In TSP, small score differences can arise from minor edge substitutions, while invalid tours should be strongly separated from valid but suboptimal tours. The strategist therefore controls the score threshold used for preference-pair construction so that the DPO update emphasizes meaningful improvements in tour quality rather than noisy or near-equivalent route variations.

The Policy Inspector (llama4:latest) acts as the reference-promotion gate for candidate TSP policies. It accepts a proposed reference only when the candidate improves verifier-based tour quality while satisfying the trajectory-conditioned proxy-KL constraint. This is important because some candidate policies can improve a small set of TSP instances by making aggressive heuristic changes, but degrade validity or generalization on other instances. The inspector therefore prevents high-variance route-generation strategies from immediately replacing the stable reference policy.

\subsubsection{Bin Packing Problem}

\textbf{Evolving Agent.}

The Evolving Agent is initialized with Llama-3.2-1B-Instruct and receives a system prompt defining it as an \textit{Evolving Agent for discrete resource-allocation problems}. For each Bin Packing instance, the policy receives item weights and a bin capacity, and it generates a structured JSON output representing an assignment of items to bins. A deterministic verifier checks that every item is assigned exactly once, no item is duplicated, and no bin exceeds the capacity constraint. The verifier then computes the number of bins used, extra bins relative to the exact integer-programming reference, and packing density. Candidate preferences are induced from this verifier score rather than from model-generated judgments.

\textbf{Role-Specialized Multi-LLM Training Supporters.}

The Exploration Supporter (gpt-oss:120b) critiques generated packing heuristics and proposes structural modifications to improve feasibility and bin utilization. For example, when the Evolving Agent relied too heavily on a simple first-fit or first-fit-decreasing rule, the supporter suggested refinements such as residual-capacity tracking, best-fit insertion, repair of overloaded bins, and reassignment of small items to improve density. These modifications help the policy escape low-density packings and reduce unnecessary bin usage.

The Fine-Tuning Strategist (Qwen 3-32b) adjusts DPO hyperparameters based on verifier-score distributions across candidate packings. In Bin Packing, many candidate outputs may be feasible but differ only slightly in density, while infeasible outputs with capacity violations must be strongly penalized. The strategist therefore modulates $\beta_{\mathrm{DPO}}$ and score thresholds to emphasize preference pairs that separate feasible, compact packings from infeasible or inefficient assignments. This prevents the Evolving Agent from overfitting to weak preferences generated by nearly equivalent packings.

The Policy Inspector (llama4:latest) enforces conservative reference promotion for Bin Packing. Because small changes in assignment logic can produce large feasibility shifts, the inspector promotes a candidate reference only if it improves verifier-based packing quality while remaining within the proxy-KL budget. In practice, this makes the inspector a safeguard against unstable heuristic changes, such as policies that reduce the number of bins on some instances but increase capacity violations or unassigned-item errors on others.

\subsection{Detailed Agent Dynamics}
\label{app:agent_dynamics_details}

Table~\ref{tab:full_agent_dynamics} summarizes the intervention frequencies of the Fine-Tuning Strategist and the acceptance behavior of the Policy Inspector across 10 independent seeds. The intervention rates are computed over active fine-tuning rounds, i.e., rounds in which the exploration stage produced candidates that triggered a fine-tuning update. Overall, the Strategist intervenes frequently across all domains, but the type of intervention varies by task. Epoch adjustment is triggered in nearly all active fine-tuning rounds across the four domains, suggesting that the training horizon required continuous adaptation as candidate quality and score distributions changed. In contrast, the $\beta_{\mathrm{DPO}}$ and score-threshold interventions are more task-dependent: Bin Packing requires the most frequent $\beta_{\mathrm{DPO}}$ adjustment, while PINN and TSP require especially frequent threshold adjustment.

The Policy Inspector shows a different pattern. Bandit has the highest reference-promotion acceptance rate ($67.5\%$), indicating that candidate references often improved evaluator score while remaining within the proxy-KL budget. PINN has a moderate acceptance rate ($37.0\%$), consistent with smoother but still selective reference progression. By contrast, TSP and Bin Packing have much lower acceptance rates ($4.0\%$ and $7.7\%$, respectively), suggesting that combinatorial heuristic updates are more volatile: proposed candidate references may improve isolated instances but often fail the score-improvement and proxy-KL gate required for reference promotion. This behavior supports the intended role of the Policy Inspector as a conservative safeguard rather than an automatic reference-refresh mechanism.

\begin{table*}[h]
\centering
\caption{Detailed activity of the Fine-Tuning Strategist and Policy Inspector. Values indicate the fraction of active fine-tuning rounds where a specific intervention was triggered.}
\label{tab:full_agent_dynamics}
\setlength{\tabcolsep}{12pt}
\begin{small}
\begin{tabular}{lccccc}
\toprule
& & \multicolumn{3}{c}{\textbf{Strategist Intervention Rate}} & \textbf{Inspector} \\
\cmidrule(lr){3-5} \cmidrule(lr){6-6}
Task & seeds & $\Delta\beta_{\mathrm{DPO}}$ & $\Delta$ Threshold & $\Delta$ Epoch & Acceptance Rate \\
\midrule
\textit{Bandit} & 10 & \textbf{67.5\%} & \textbf{49.1\%} & \textbf{99.1\%} & \textbf{67.5\%} \\
\midrule
\textit{PINN} & 10 & \textbf{49.5\%} & \textbf{88.7\%} & \textbf{99.0\%} & \textbf{37.0\%} \\
\midrule
\textit{TSP} & 10 & \textbf{40.7\%} & \textbf{91.9\%} & \textbf{97.2\%} & \textbf{4.0\%} \\
\midrule
\textit{Bin Packing} & 10 & \textbf{92.0\%} & \textbf{35.9\%} & \textbf{95.0\%} & \textbf{7.7\%} \\
\bottomrule
\end{tabular}
\end{small}
\end{table*}

Table~\ref{tab:full_agent_dynamics} also highlights that ATLAS does not apply the same control policy uniformly across domains. In the Bandit task, both $\beta_{\mathrm{DPO}}$ adjustment and reference promotion are frequent, reflecting the need to track a drifting reward model while still allowing stable reference progression. In PINN, the high threshold-adjustment rate suggests that preference-pair selection must be actively controlled as the optimization landscape becomes stiff. In TSP and Bin Packing, the low acceptance rates indicate that the inspector rejects most candidate references, which is expected when discrete heuristic changes cause high-variance performance shifts across instances.

The number of active fine-tuning rounds may vary across seeds because ATLAS skips expensive fine-tuning when the exploration phase fails to produce candidates with sufficient evaluator-score improvement. This mechanism reduces unnecessary updates and prevents the preference dataset from being dominated by low-quality or uninformative candidate pairs. Consequently, the reported intervention rates should be interpreted as conditional rates over active fine-tuning rounds rather than over all exploration rounds.

\subsection{Hyperparameters, Hardware, and Timing}
\label{appendix:param}
\textbf{Hyperparameters.} Lists the hyperparameters used for training, and task-specific configurations are summarized in Table \ref{tab:hyperparams}. 
\begin{table}[h] 
\centering 
\begin{tabular}{lll} 
\toprule 
\textbf{Category} & \textbf{Parameter} & \textbf{Value} \\
                 \midrule 
General          & Total Rounds & $501$ (Bandit); $201$ (PINN) \\ 
                 & Number of Seeds & $10$ (Independent Runs) \\
                 & Sampling Strategy & Continuous ($K=20$ rounds) \\
                 & Number of Islands & $6$ \\
                 \midrule 
DPO Training     & Finetuning Frequency & Every $20$ rounds \\ 
                 & Initial Learning Rate & $5 \times 10^{-6}$ \\
                 & Initial Beta ($\beta_{DPO}$) & $0.6$ \\
                 & Initial Epochs per Update & $1$ \\
                 & LoRA Rank / Alpha & $r=64, \alpha=32$ \\
                 & KL Thresholds ($\delta_{High}$) & $0.002$ \\
                 & Score improvement tolerance ($\epsilon_s$) & $0.0007$ \\
                 \midrule
LLM Temperature  & Llama-3.2-1B-Instruct & 0.9 \\
                 & Gpt-oss: 120b & 0.7 \\
                 & Qwen 3-32b & 0.7 \\
                 & Llama4:latest & 0.7 \\
                 \midrule 
Bandit Task      & Number of arms ($k$) & 5 \\
                 & Dimensions ($d$) & 5 \\ 
                 & Horizon ($H$) & 2000 \\ 
                 & Total variation budget ($V_T$) & $8000$ \\ 
                 & Drift limits ($\delta_{min}, \delta_{max}$) & $1.0, 5.0$ \\
                 \midrule
PINN Task        & Initial Viscosity ($\nu$) & $0.01/\pi$ \\
                 & Sampling ($N_{coll}, N_{ic}, N_{bc}$) & $512, 64, 64$ \\
                 & Inner Optimizer & Adam \\
                 & Inner Learning Rate & $1 \times 10^{-4}$ \\
                 & Inner Optimization Steps & $500$ \\
                 \midrule
TSP Task         & Instance Generation & Uniform nodes in $[0,1]^2$ \\
                 & Problem Sizes ($n$) & $10, 20, 50, 100$ nodes \\
                 & Distance Metric & Euclidean distance \\
                 & Output Format & JSON tour \\
                 & Maximum Refinement Rounds & $5$ \\
                 & Exact Evaluation Baseline & Exact TSP solver \\
                 & Primary Metric & Relative improvement \\
                 \midrule
Bin Packing Task & Instance Generation & Random item weights \\
                 & Problem Sizes ($n$) & $10, 50, 100$ items \\
                 & Bin Capacity ($C$) & Fixed per instance \\
                 & Output Format & JSON bin assignment \\
                 & Maximum Refinement Rounds & $5$ \\
                 & Exact Evaluation Baseline & Integer-programming solver \\
                 & Primary Metric & Relative improvement \\
                 \bottomrule
\end{tabular} 
\caption{Hyperparameters for Training, DPO, and Tasks.} 
\label{tab:hyperparams} 
\end{table}

\textbf{Hardware.} For benchmarking each task, we were using 
8 Nvidia l40s GPUs with 20 GB of RAM. From this cluster, we specifically allocated our Evolving Agent (Llama-3.2-1B-Instruct), which used 1 Nvidia l40s GPU with 20 GB of RAM and 12 CPU cores.

\textbf{Timing.} A full run of ATLAS, EvoDPO, and evoTune takes approximately 2 hours per seed for the Bandit task and approximately 1 hour per seed for the PINN task with the provided hardware environment. The runtime for other methods and/or benchmarks falls within a similar range.

\paragraph{External assets, models, and solvers.}
All experiments use publicly described or API-accessed LLM backbones and standard optimization or evaluation software. We do not redistribute model weights, third-party solver code, or proprietary API assets. The LLM backbones used in our experiments are Llama-3.2-1B-Instruct for the Evolving Agent, gpt-oss-120B for the Exploration Supporter, Qwen 3-32b for the Fine-Tuning Strategist, and Llama4-latest for the Policy Inspector. For combinatorial evaluation, exact or integer-programming solvers are used only to compute reference objective values and feasibility checks; they do not provide candidate solutions to ATLAS. These external tools are used under their respective access terms and are listed here solely for reproducibility.


\subsection{KL Proxy and Limitations}
\label{app:klproxy}
The gate uses a trajectory-conditioned token-level KL proxy (Eq.~(3)), computed from next-token distributions under teacher forcing.
This proxy is efficient and targets local drift in autoregressive behavior, but it is not equal to unconditional sequence-level KL and can exhibit variance depending on the chosen trajectory and the sampled gate subset $G_k$.

\subsection{Selection Adaptivity}
\label{app:robust}
Because candidates and scores are generated on-policy, gate decisions can be adaptive to the current phase buffer. We mitigate this by evaluating the gate on a fresh randomly sampled subset $G_k$ each phase, reducing reuse of the same examples for candidate
selection and acceptance decisions.

\subsection{Implementation Details}

\begin{table}[!ht]
\centering
\caption{Key software libraries and versions.}
\label{tab:software}
\begin{small}
\begin{tabular}{ll}
\toprule
\textbf{Category} & \textbf{Libraries \& Versions} \\
\midrule
\textbf{Core Frameworks} & Python 3.10 \\
& PyTorch 2.8.0 \\
& CUDA 12.6.0 \\
& JAX 0.4.30 \\
\midrule
\textbf{LLM \& Optimization} & Transformers 4.57.3 \\
& TRL (Transformer Reinforcement Learning) 0.24.0 \\
& PEFT 0.17.1 \\
& Accelerate 1.10.1 \\
& Tokenizers 0.22.1 \\
& OpenAI 2.8.1 \\
\midrule
\textbf{Scientific Computing} & NumPy 1.26.4 \\
& SciPy 1.13.1 \\
& SymPy 1.14.0 \\
& Pandas 2.3.3 \\
\bottomrule
\end{tabular}
\end{small}
\end{table}
Our experiments were implemented using PyTorch and Hugging Face Transformers for LLM training. Key libraries and their versions are listed in Table \ref{tab:software}.
\FloatBarrier

\section{Limitations, Broader Impact, and Ethics}

\paragraph{Limitations.}
Our theoretical analysis is intentionally stylized and isolates adaptive reference progression in a non-stationary preference contextual bandit. It does not provide an end-to-end guarantee for full sequence-level LLM fine-tuning, and it does not directly model the finite checkpoint candidate set, the empirical token-level KL proxy, or the discrete accept/reject behavior of the inspector gate. Empirically, our experiments focus on four executable optimization domains: non-stationary contextual bandits, PINN loss reweighting, TSP, and Bin Packing. Although these tasks cover decision-making, scientific machine learning, and combinatorial optimization, they do not exhaustively represent all open-ended scientific-discovery settings.

\paragraph{Broader impact.}
ATLAS is intended to improve evaluator-driven self-improvement for executable optimization tasks, including scientific computing and algorithmic search. Potential positive impacts include reducing manual effort in designing optimization heuristics, stabilizing long-horizon preference-based training, and improving scientific-computing workflows where evaluator feedback is available. Potential negative impacts include over-reliance on automatically generated code, propagation of evaluator misspecification, generation of incorrect or inefficient programs, and misuse of autonomous code-improvement pipelines in settings where generated code could affect safety-critical systems. We mitigate these risks in our experiments by using task-specific executable evaluators, fixed external validation metrics for PINN evaluation, and a conservative KL-based reference-promotion gate. The proposed system should not be deployed in safety-critical settings without human review, stronger sandboxing, and domain-specific validation.

\paragraph{Ethics and data.}
The experiments do not involve human subjects, private personal data, crowdsourcing, or scraped user datasets. All preference labels are induced by executable task evaluators rather than human annotators. The work uses existing LLMs and software frameworks; their licenses and access terms should be respected when reproducing or extending the system.

\section{Prompt Templates for Code Generation}
\label{prompt}
This section provides templates for the given prompts for LLM agents. All tasks use the same Fine-Tuning Strategist and Policy Inspector prompt, whereas Exploration Supporter and Evolving Agent use suitable prompts for each task.

\subsection{Problem Description: Non-stationary Contextual Bandit}

\begin{lstlisting}[style=promptstyle, literate={_}{\_}1 {π}{{$\pi$}}1 {θ}{{$\theta$}}1 {∈}{{$\in$}}1]
Non-stationary linear contextual bandit: K arms, d-dim contexts, drifting θ(t).
At each t, you observe context matrix context ∈ R^{Kxd} (one row per arm).
Reward is generated by the evaluator: r_t = context[a_t] * θ(t) + noise.

Goal: minimize regret. You only choose a_t using history (no access to θ or expected rewards).

The evaluator maintains a sliding-window ridge estimator:
  A_t = sum_{tau in W_t} phi_tau phi_tau^T + lambda I
  b_t = sum_{tau in W_t} h_tau phi_tau
and exposes it to your policy as:
  - policy.A (dxd), policy.b (d,)
Your policy must tune hyperparameters by setting:
  - policy.lambda_reg  (ridge regularizer, float > 0)
  - policy.window_size (sliding window size W, int >= 1)
\end{lstlisting}

\subsection{Problem Description: PINN Burgers}
\label{prob_pinn} 
\begin{lstlisting}[style=promptstyle,literate={_}{\_}1{:}{:}1]
We are evolving training loss functions for a physics-informed neural network (PINN) that solves the 1D viscous Burgers equation: u_t + u * u_x - nu * u_xx = 0,  (t, x) in [0, 1] x [-1, 1] with:
  - Initial condition:  u(0, x) = -sin(pi * x)
  - Boundary conditions: u(t, -1) = 0, u(t, 1) = 0

The candidate function pinn_loss_strategy(residuals, ic_err, bc_err, extras) receives:
  - residuals: squared PDE residuals at N_coll collocation points
  - ic_err : squared initial-condition errors at N_ic points
  - bc_err : squared boundary-condition errors at 2*N_bc points
  - extras : metadata such as 'epoch' and 'nu'

The candidate function must combine these quantities into a scalar training loss tensor used only for inner PINN optimization.

Candidate-generated losses are not used directly as the evaluation score. After training a PINN with the candidate loss, an external fixed evaluator computes validation error using evaluator-defined terms:
  - normalized PDE residual error,
  - initial-condition error,
  - boundary-condition error,
  - solution L2 error against a high-resolution reference solution.

The scalar score is the negative of this fixed evaluator error:
  score = -fixed_evaluator_error

Thus:
  - Lower fixed evaluator error gives a larger score.
  - Higher fixed evaluator error gives a smaller score.

The goal is to design training loss strategies that lead to low fixed-evaluator error after PINN optimization. The generated loss may shape training, but all preference labels and final metrics are computed using the fixed external evaluator.
\end{lstlisting}

\subsection{Problem Description: Traveling Salesperson Problem}
\label{prob_tsp}

\begin{lstlisting}[style=promptstyle,literate={_}{\_}1]
Traveling Salesperson Problem (TSP): You are given n nodes and a distance matrix D.
D[i][j] is the distance from node i to node j.

Goal:
Find a valid tour that:
  - starts at node 0,
  - visits every node exactly once,
  - returns to node 0,
  - minimizes the total tour length.

A candidate tour is evaluated by an external verifier. The verifier checks:
  - whether every node is visited exactly once,
  - whether the tour starts and ends at node 0,
  - whether the output is valid JSON,
  - the total tour length.

The score is based on the tour length and feasibility. Invalid tours are penalized.
The model must output a machine-readable JSON object so the verifier can parse it.
\end{lstlisting}

\subsection{Problem Description: Bin Packing}
\label{prob_binpacking}

\begin{lstlisting}[style=promptstyle,literate={_}{\_}1]
Bin Packing Problem (BPP): You are given n items with positive weights and a fixed bin capacity C.

Goal:
Assign every item to a bin while:
  - using as few bins as possible,
  - never exceeding the capacity C in any bin,
  - assigning every item exactly once,
  - avoiding duplicated or missing items.

A candidate packing is evaluated by an external verifier. The verifier checks:
  - whether every item is assigned exactly once,
  - whether any bin exceeds capacity,
  - whether the output is valid JSON,
  - the number of bins used,
  - the packing density.

The score is based on feasibility and the number of bins used. Invalid packings are penalized.
The model must output a machine-readable JSON object so the verifier can parse it.
\end{lstlisting}

\subsection{Evolving Agent: Non-stationary Contextual Bandit}

\label{prob_bandit:student} 
\begin{lstlisting}[style=promptstyle,literate={_}{\_}1]
You are an Evolving Agent in online learning and contextual bandits.
TASK: Write a POLICY function for a Non-Stationary Linear Contextual Bandit.
Define exactly ONE function:
    def policy(context, history_contexts, history_actions, history_rewards, t, rnd) -> int
The evaluator (not you) runs the environment and scoring.

RULES:
1. Use ONLY `import numpy as np` (never `import random`)
2. Use ONLY the provided `rnd` for randomness (do not create your own RandomState)
3. Do NOT reference global K/d. Always derive them as: K, d = context.shape
4. Return a single int arm index in [0, K-1]
5. history_contexts/history_actions/history_rewards are NumPy arrays (read-only inputs). Do NOT call .append() on them.
6. You MUST return an int on ALL code paths (no missing return / never return None)
8. You tune hyperparameters using feedback from teacher:
     policy.window_size (int >= 1)
     policy.lambda_reg  (float > 0)
     policy.ucb_alpha   (float >= 0)
9. You must NOT simulate $\theta$, rewards, contexts, or regret (evaluator does this).
10. Do NOT import sklearn, scipy, torch, pandas, or any external library. Only `import numpy as np` is allowed.
11. No markdown, only a single function definition.

Define exactly ONE function:
    def policy(context, history_contexts, history_actions, history_rewards, t, rnd) -> int
    
PROBLEM SETUP (Evaluator-owned)
At each time step t:
  1) You observe context $\in R^{K\times d}$ (one row per arm).
  2) You choose an arm a_t using ONLY history:
       history_contexts[:t], history_actions[:t], history_rewards[:t].
       IMPORTANT: history_* are NumPy arrays; do NOT modify them and do NOT use .append(). Use slicing/indexing only.
  3) The evaluator generates reward and regret (you never see $\theta$ or expected rewards).
  
The evaluator maintains ridge regression state and exposes it to you as:
  - policy.A (d x d matrix)
  - policy.b (d vector)
Your policy can compute:
  theta_hat = np.linalg.solve(policy.A, policy.b)
  
Return ONLY the function definition. No test code.
\end{lstlisting}

\subsection{Evolving Agent: PINN Burger's equation}
\label{prob_pinn:student} 
\begin{lstlisting}[style=promptstyle]
You must define exactly one function:
    def pinn_loss_strategy(residuals, ic_err, bc_err, extras):
Arguments:
- residuals: 1D PyTorch tensor of shape [N_coll],
             containing squared PDE residuals at collocation points.
- ic_err   : 1D PyTorch tensor of shape [N_ic],
             containing squared initial condition errors.
- bc_err   : 1D PyTorch tensor of shape [2*N_bc],
             containing squared boundary condition errors (left and right).
- extras   : dictionary with metadata (e.g., 'epoch', 'nu', etc.).
Return:
- A scalar PyTorch tensor (a single number in a torch.Tensor).
- DO NOT return a function or a callable object.
- Example: return residuals.mean() + ic_err.mean()  -> CORRECT
- Example: return lambda: residuals.mean()  -> WRONG (returns a function)
- Example: DO NOT define inner functions and return them.

Constraints:
- Use ONLY PyTorch tensor operations (no NumPy inside the function).
- The loss MUST be a scalar tensor (0-dimensional).
- Ensure all loss weights (lambdas) are between 0.001 and 5.0 to avoid ignoring constraints or gradient explosion.
- Avoid in-place operations that could break autograd.
- Use simple operations: .mean(), .sum(), elementwise arithmetic,
  possible time-dependent weights based on extras['epoch'], etc.
- Keep the computation O(N) in the number of samples.

Goal:
- Design a loss that leads the PINN to satisfy the Burgers PDE as well as
  the initial and boundary conditions. You may:
  - Re-weight PDE vs IC vs BC errors.
  - Use adaptive weights that change with epoch.
  - Emphasize harder points (e.g., near shocks) by using powers
    or normalized residuals.

Return JUST the function definition. Do not include training code.
\end{lstlisting}

\subsection{Evolving Agent: Traveling Salesperson Problem}
\label{prob_tsp:student}

\begin{lstlisting}[style=promptstyle,literate={_}{\_}1]
You are an Evolving Agent for combinatorial optimization.

TASK:
Generate a valid Traveling Salesperson Problem (TSP) tour.

You are given:
- n: number of nodes
- distance_matrix: a 2D list or array where distance_matrix[i][j] is the distance from node i to node j

Output:
Return ONLY a valid JSON object with the following structure:
{
  "tour": [0, ..., 0]
}

Rules:
1. The tour MUST start at node 0 and end at node 0.
2. Every node from 0 to n-1 MUST appear exactly once before returning to node 0.
3. Do NOT duplicate nodes except for the final return to node 0.
4. Do NOT omit any node.
5. Do NOT include nodes outside the range [0, n-1].
6. The output MUST be valid JSON.
7. Do NOT include markdown, comments, or extra text outside the JSON object.
8. The evaluator will compute the total tour length. Do NOT self-report the score.

Goal:
Produce a feasible tour with small total distance.
You may use constructive heuristics such as nearest neighbor, insertion, or local repair,
but the final answer must be only the JSON object.
\end{lstlisting}

\subsection{Evolving Agent: Bin Packing}
\label{prob_binpacking:student}

\begin{lstlisting}[style=promptstyle,literate={_}{\_}1]
You are an Evolving Agent for discrete resource allocation.

TASK:
Generate a feasible bin packing assignment.

You are given:
- capacity: maximum capacity of each bin
- item_weights: a list of positive item weights

Output:
Return ONLY a valid JSON object with the following structure:
{
  "bins": [[...], [...], ...]
}

Each inner list contains item indices assigned to that bin.

Rules:
1. Every item index from 0 to len(item_weights)-1 MUST appear exactly once.
2. Do NOT duplicate any item.
3. Do NOT omit any item.
4. The total weight in each bin MUST be less than or equal to capacity.
5. Do NOT include item indices outside the valid range.
6. The output MUST be valid JSON.
7. Do NOT include markdown, comments, or extra text outside the JSON object.
8. The evaluator will compute feasibility, number of bins, and packing density. Do NOT self-report the score.

Goal:
Use as few bins as possible while satisfying all capacity constraints.
You may use constructive heuristics such as first-fit decreasing, best-fit decreasing,
or local repair, but the final answer must be only the JSON object.
\end{lstlisting}

\subsection{Exploration Supporter: Non-stationary Contextual Bandit}
\label{prob_bandit:gptoss} 
\begin{lstlisting}[style=promptstyle]
You are a bandit hyperparameter tuner. Give BRIEF, CODE-READY feedback.

Rules:
- Keep feedback under 120 words total
- Use numpy only, keep it simple
- You must change assignments to:
    policy.lambda_reg and policy.window_size
- Do NOT add UCB bonuses, refactor logic, or simulate the environment.
- Put your fix inside an 'if t == 0' block.

Format (STRICT):
Issue: One sentence describing the main problem.
- Fix: Hyperparameter assignments only
- Why: One sentence explaining why this helps.

\end{lstlisting}

\subsection{Exploration Supporter: PINN Burger's equation}
\label{prob_pinn:gptoss} 
\begin{lstlisting}[style=promptstyle]
You are an Evolving Agent in physics-informed neural networks (PINNs), PDE-constrained optimization, and training stability for stiff nonlinear PDEs (especially 1D viscous Burgers).

The student is evolving loss functions for a Burgers PINN of the form u_t + uu_x - vu_xx = 0 on a bounded space-time domain with initial and boundary conditions.

You will receive:
- The candidate loss function (Python)
- Brief training diagnostics (e.g., residual plots, loss curves, or short logs).

Your job is to act as a strict but constructive teacher focused ONLY on the loss design and its training dynamics - do NOT propose architectural changes (no new networks, no new optimizers, etc.).

You can adjust only:
- Loss terms and their weights,
- Normalization / scaling strategies,
- Sampling strategies in the loss (e.g., where to place collocation points),
- Simple curriculum or annealing schedules for the loss weights.
\end{lstlisting}

\subsection{Exploration Supporter: Traveling Salesperson Problem}
\label{prob_tsp:gptoss}

\begin{lstlisting}[style=promptstyle,literate={_}{\_}1]
You are a combinatorial-optimization feedback agent for TSP candidate programs.

You will receive:
- A generated TSP solution strategy or candidate output
- Verifier feedback such as invalid nodes, duplicated nodes, omitted nodes, tour length, or optimality gap

Your job:
Give brief, code-ready feedback that improves feasibility and tour quality.

Focus on:
- Repairing duplicated or omitted nodes
- Ensuring the tour starts and ends at node 0
- Reducing tour length using simple local improvements
- Suggesting nearest-neighbor, insertion, 2-opt, or edge-crossing repair when appropriate

Rules:
- Keep feedback concise and actionable
- Do NOT claim a score without verifier output
- Do NOT propose external libraries
- Do NOT change the required JSON output format

Format:
Issue: one sentence describing the main problem.
Fix: one or two concrete changes to the heuristic or output construction.
Why: one sentence explaining why this improves feasibility or tour quality.
\end{lstlisting}

\subsection{Exploration Supporter: Bin Packing}
\label{prob_binpacking:gptoss}

\begin{lstlisting}[style=promptstyle,literate={_}{\_}1]
You are a combinatorial-optimization feedback agent for Bin Packing candidate programs.

You will receive:
- A generated bin-packing strategy or candidate output
- Verifier feedback such as capacity violations, duplicated items, omitted items, bin count, or packing density

Your job:
Give brief, code-ready feedback that improves feasibility and bin utilization.

Focus on:
- Assigning every item exactly once
- Repairing overloaded bins
- Reducing the number of bins
- Improving packing density using first-fit decreasing, best-fit decreasing, residual-capacity tracking, or local reassignment

Rules:
- Keep feedback concise and actionable
- Do NOT claim a score without verifier output
- Do NOT propose external libraries
- Do NOT change the required JSON output format

Format:
Issue: one sentence describing the main problem.
Fix: one or two concrete changes to the packing heuristic or repair rule.
Why: one sentence explaining why this improves feasibility or bin utilization.
\end{lstlisting}

\subsection{Fine-Tuning Strategist}
\label{advisor:qwen} 
\begin{lstlisting}[style=promptstyle]
You are a careful training-strategy advisor for a self-evolving code agent using
Direct Preference Optimization (DPO).

You receive a small JSON blob describing the current state of training
(round number, score thresholds, dataset sizes, etc.).
You must respond with a SMALL JSON object suggesting hyperparameter tweaks.

Constraints:
- You may suggest new values for:
    - "score_threshold": float
    - "beta": float
    - "num_train_epochs": int
- Only change a parameter if there is a clear benefit.
  Otherwise, OMIT that key and leave it unchanged.
- Optionally include a short comment string explaining your reasoning.

Return only a JSON object, no extra text.
\end{lstlisting}
\lstdefinestyle{prompt_top}{
    style=promptstyle,
    frame=top|left|right, 
    belowskip=0pt         
}

\lstdefinestyle{prompt_bottom}{
    style=promptstyle,
    frame=bottom|left|right, 
    aboveskip=0pt            
}
\subsection{Policy Inspector}
\label{advisor:llama4} 
\begin{lstlisting}[style=prompt_bottom, literate={_}{\_}1 {π}{{$\pi$}}1 {θ}{{$\theta$}}1 {∈}{{$\in$}}1 {δ}{{$\delta$}}1]
Policy inspector for Evolving DPO reference update.

CONTEXT:
- The candidate policy π^kl_t was trained via DPO with loss L_DPO(π; π_ref,t)
- Per-step KL divergence D(π^kl_t || π_ref,t) has ALREADY been checked algorithmically
- The kl_divergence field shows how much the policy drifted from the previous reference
- If KL exceeded the threshold δ, this inspector would NOT have been called (auto-reject)

Decision rule (based on improvement signal KL divergence D(π^kl_t || π_ref,t)):
- If improvement_signal >= epsilon and improvement_signal <= KL_THRESHOLD_HIGH: ACCEPT
- If improvement_signal < epsilon or improvement_signal > KL_THRESHOLD_HIGH: REJECT

DEFAULT: ACCEPT. The KL divergence check has already passed.

Output: JSON only, no other text. {"accept": true, "reason": "brief reason"}
\end{lstlisting}

\end{document}